%% file: main.tex
\documentclass[10pt,twocolumn,letterpaper, tikz]{article}

\usepackage[accsupp]{axessibility}  
\usepackage[final]{cvpr}              

\definecolor{cvprblue}{rgb}{0.21,0.49,0.74}
\usepackage[pagebackref,breaklinks,colorlinks,allcolors=cvprblue]{hyperref}

\input{utils/math_commands}

\usepackage{hyperref}
\usepackage{url}
\usepackage{tikz}

\usepackage[utf8]{inputenc} 
\usepackage[T1]{fontenc}    
\usepackage{hyperref}       
\usepackage{url}            
\usepackage{booktabs}       
\usepackage{amsfonts}       
\usepackage{nicefrac}       
\usepackage{microtype}      
\usepackage{xcolor}         
\usepackage{amsmath}
\usepackage{amsthm}
\usepackage{amssymb}

\usepackage{graphicx}
\usepackage{colortbl}
\usepackage{tabulary}
\usepackage{adjustbox}
\usepackage{multirow}
\usepackage{makecell} 
\usepackage{enumerate}
\usepackage{enumitem}
\usepackage{subcaption}
\usepackage{algorithm}
\usepackage[noend]{algpseudocode}
\usepackage{caption}
\usepackage{setspace}       
\usepackage{pifont}

\usepackage{wrapfig}
\usepackage{stackengine}
\usepackage{comment}
\usepackage{mathtools}
\usepackage{multirow}
\usepackage[capitalize,noabbrev]{cleveref}

\newtheorem{theorem}{Theorem}[section]

\newtheorem{lemma}[theorem]{Lemma}

\definecolor{myblue}{HTML}{0153d6} 
\definecolor{myred}{HTML}{ff409b}

\definecolor{commentcolor}{RGB}{110, 154, 155}

\newcommand{\oursfull}{\textnormal{Attention Modulated Overshooting sampler}}
\newcommand{\ours}{\textnormal{AMO}}
\newcommand{\oursbf}{\textbf{AMO}}
\newcommand{\ourswoatt}{\textnormal{Overshooting sampler}}

\newcommand{\xmark}{\ding{55}}%

\newcommand{\atphantom}{\vphantom{${}^2$}}
\newcommand{\AProcedure}[2]{\Procedure{\smash{#1}}{\smash{#2}}}

\newcommand{\AState}[1]{\State{\smash{#1}}}
\newcommand{\AFor}[1]{\For{\smash{#1}}}

\newcommand{\dd}{\boldsymbol{d}}
\newcommand{\mm}{\boldsymbol{m}}

\newcommand{\xxi}{\boldsymbol{\xi}}

\newcommand\myapprox{\mathrel{\overset{\makebox[0pt]{\mbox{\normalfont\tiny\sffamily Law}}}{\approx}}}

\newcommand\myeq{\mathrel{\overset{\makebox[0pt]{\mbox{\normalfont\tiny\sffamily Law}}}{=}}}
\def\paperID{5812} 
\def\confName{CVPR}
\def\confYear{2025}

\title{AMO Sampler: Enhancing Text Rendering with Overshooting}

\vspace{-20pt}
\author{
Xixi Hu$^{1,2}$\thanks{Equal contribution}, Keyang Xu$^{1}$\footnotemark[1], Bo Liu$^{2}$, Qiang Liu$^{2}\thanks{Qiang Liu and Hongliang Fei jointly advised this work.}$ and Hongliang Fei$^{1}$\footnotemark[2]\\
$^1$Google, $^2$University of Texas at Austin \\
{\tt\small \{hxixi,bliu,lqiang\}@cs.utexas.edu; \{keyangxu,hongliangfei\}@google.com}
}

\begin{document}

\twocolumn[{
\maketitle
\vspace{-20pt}
\begin{center}
    \captionsetup{type=figure}
    \includegraphics[width=0.92\textwidth]{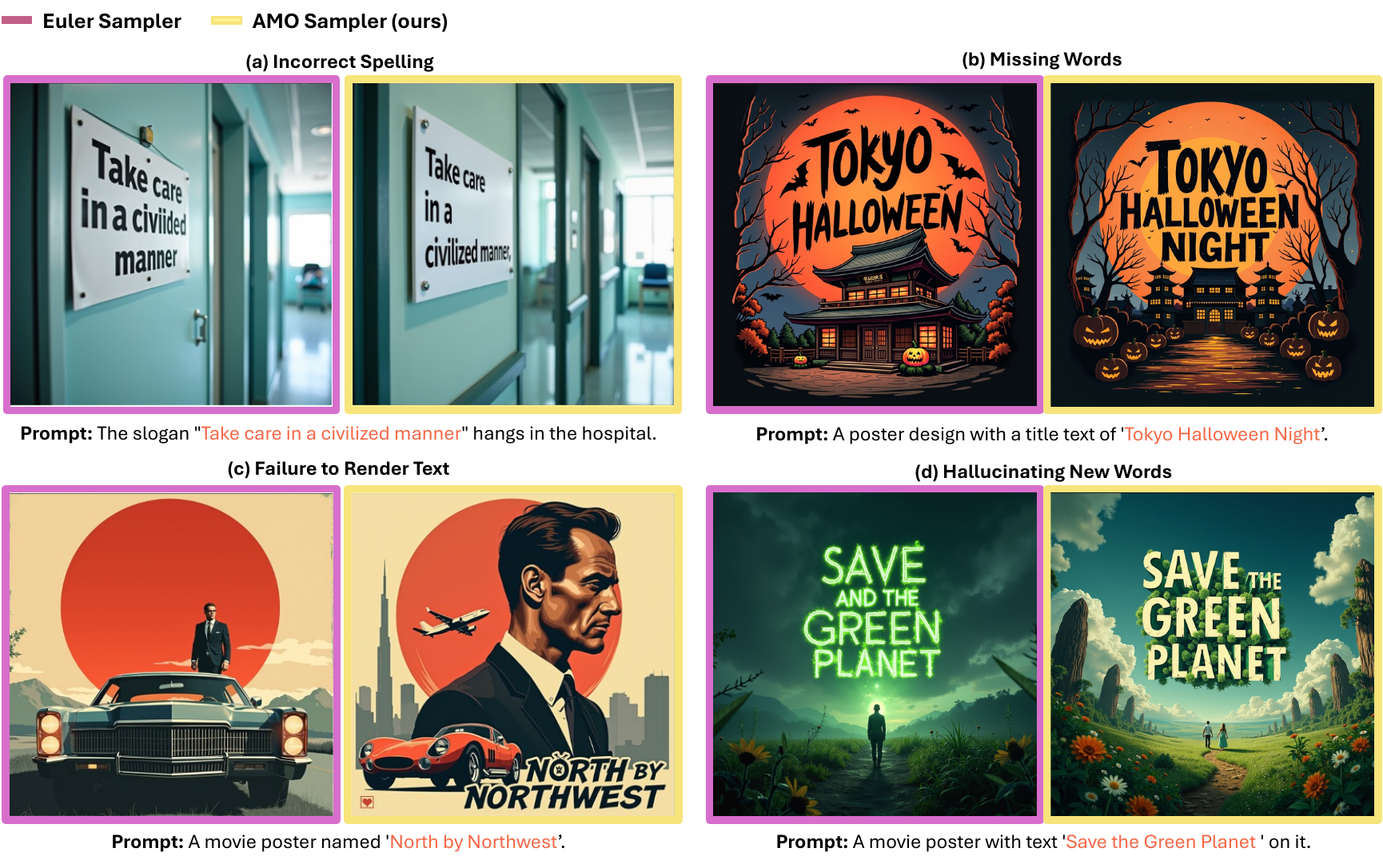}
    \vspace{-5pt}
    \caption{\textbf{Improved Text Rendering.} \textbf{(a)-(d)} illustrate four common text rendering mistakes in text-to-image generations. Compared to the standard Euler sampler \textbf{(purple)}, our \oursfull{}~(\ours) \textbf{(yellow)} produces accurate and complete text \emph{without} additional training, and remains as computationally efficient as the Euler sampler.
    }
    \label{fig:teasing}
\end{center}
}]

\maketitle
\renewcommand{\thefootnote}{} 
\footnotetext{$\ast$ Equal contribution. }
\footnotetext{$\dagger$ Qiang Liu and Hongliang Fei jointly advised this work.}
\input{sections/0_abstract}

\input{sections/1_intro}

\input{sections/2_background}

\input{sections/3_method}

\input{sections/4_related_work}

\input{sections/5_exp}
\input{sections/6_conclusion}

\clearpage
\input{sections/7_acknowledgement}
{
    \small
    \bibliographystyle{ieeenat_fullname}
    \bibliography{main}
}
\clearpage

\appendix
\input{sections/7_appendix}

\end{document}

%% file: utils/math_commands.tex
\usepackage{amsmath,amsfonts,bm}

\def\Eqref#1{Equation~\eqref{#1}}

\def\1{\bm{1}}

\def\vv{{\bm{v}}}

\def\vx{{\bm{x}}}

\DeclareMathAlphabet{\mathsfit}{\encodingdefault}{\sfdefault}{m}{sl}
\SetMathAlphabet{\mathsfit}{bold}{\encodingdefault}{\sfdefault}{bx}{n}

\newcommand{\E}{\mathbb{E}}

\newcommand{\R}{\mathbb{R}}

\renewcommand{\vec}[1]{{\boldsymbol{#1}}}
\newcommand{\X}{\vec{X}}
\newcommand{\Z}{\vec{Z}}
\newcommand{\mat}[1]{\mathbf{#1}}

%% file: sections/0_abstract.tex
\begin{abstract}
Achieving precise alignment between textual instructions and generated images in text-to-image generation is a significant challenge, particularly in rendering written text within images. Sate-of-the-art models like Stable Diffusion 3 (SD3), Flux, and AuraFlow still struggle with accurate text depiction, resulting in misspelled or inconsistent text.
We introduce a training-free method with \emph{minimal} computational overhead that \emph{significantly} enhances text rendering quality.
Specifically, we introduce an overshooting sampler for pretrained rectified flow (RF) models, by alternating between over-simulating the learned ordinary differential equation (ODE) and reintroducing noise. Compared to the Euler sampler, the overshooting sampler effectively introduces an extra Langevin dynamics term that can help correct the compounding error from successive Euler steps and therefore improve the text rendering. However, when the overshooting strength is high, we observe over-smoothing artifacts on the generated images. 
To address this issue, we propose an \oursfull{} (\ours{}), which adaptively control the strength of overshooting for each image patch according to their attention score with the text content. 
\ours{} demonstrates a \textbf{32.3\%} and \textbf{35.9\%} improvement in text rendering accuracy on SD3 and Flux without compromising overall image quality or increasing inference cost. Code available at: \href{https://github.com/hxixixh/amo-release}{https://github.com/hxixixh/amo-release}. 
\end{abstract}

%% file: sections/1_intro.tex
\section{Introduction}
\label{sec::intro}
Recent advances in diffusion models~\citep{song2020denoising,song2020score, song2019generative, liu2022rectified, lipman2022flow} have enabled high-quality image and video generation. Text-to-image generation, where neural networks create images from natural language prompts, has emerged as a transformative application of AI. Despite significant progress, a key challenge remains in precisely aligning generated images with given text instructions, especially in text rendering tasks, where models often struggle to display specific text accurately. This misalignment results in errors like misspellings or incorrect wording (see Figure~\ref{fig:teasing} for examples), limiting the models' utility in fields like graphic design, advertising, and assistive technologies~\cite{saharia2022photorealistic, esser2024scaling}.

Although fine-tuning with curated text data can improve text rendering~\citep{chen2023textdiffuser, chen2024textdiffuser}, it requires additional data collection and computationally expensive retraining, making it impractical for many applications. Furthermore, such fine-tuning may inadvertently compromise the model's overall image-generation capabilities. In this work, we investigate methods to enhance text rendering quality. We focus on rectified flow (RF)~\citep{liu2022rectified} models, which have emerged as a compelling alternative to conventional diffusion models due to their conceptual simplicity, ease of implementation, and improved generation quality~\citep{esser2024scaling}. Specifically, we introduce a lightweight, training-free sampling approach that significantly improves text rendering accuracy in generated images.

We introduce a novel and straightforward stochastic sampling approach on top of RF models, named the \ourswoatt{}, that iteratively adds noise to the Euler sampler while preserving the marginal distribution. In particular, the \ourswoatt{} alternates between over-simulating the learned ordinary differential equation (ODE) and re-introducing the noise (See Section~\ref{sec::method}). As we will show in Section~\ref{sec::method-overshoot}, \ourswoatt{} effectively introduces an extra Langevin dynamics term that can help correct the compounding error from the successive applications of the Euler sampler, therefore enhancing the text generation quality. The overshooting strength is controlled by a hyperparameter $c > 0$, corresponding to the magnitude of the Langevin step. When $c$ is large, the Langevin term becomes inaccurate and can introduce error itself. To mitigate this problem, we propose a targeted use of the \ourswoatt{} for text rendering, by adaptively controlling its strength on different patches of the image according to their attention scores with the text content in the prompt. We name the combined approach as \oursfull~(\ours{}).

We validate the \ours{} sampler on state-of-the-art RF-based text-to-image models, including SD3, Flux, and AuraFlow. Our experiments demonstrate a \emph{significant} improvement in text rendering accuracy, with correct text generation rates increasing by \textbf{32.3\%} on SD3 and \textbf{35.9\%} on Flux, \emph{without} compromising overall image quality.

%% file: sections/2_background.tex
\section{Background on Rectified Flow}
\label{sec::background}
This section provides a brief introduction to Rectified Flow (RF)~\citep{liu2022rectified}. RF seeks to learn a mapping from an easy-to-sample initial distribution \( \X_0 \sim \pi_0 \), 
which we assume to be the standard Gaussian $\mathcal{N}(\vec 0, \mat I)$,  to a target data distribution \( \X_1 \sim \pi_1 \). This is achieved by learning a velocity field \( \vec{v} \) that minimizes the following objective:
\[
\min_{\vec v} \int_0^1 \mathbb{E}_{(\X_0, \X_1) \sim \pi_0, \times \pi_1} \left[ \left\| v(\X_t, t) - 
\dot{\X}_t  \right\|^2 \right] dt.
\]
In RF, \( \X_t \) is defined as a time-differentiable interpolation between \( \X_0 \) and \( \X_1 \), i.e., \( \X_t = t \X_1 + (1 - t) \X_0 \) and \( \dot{\X}_t = \X_1 - \X_0 \). Once \( \vec{v} \) is learned, it induces an ordinary differential equation (ODE):
\[
\frac{d}{d t} \Z_t = \vec{v}(\Z_t, t), ~~~\forall t \in[0,1], \quad \Z_0 = \X_0.
\]
It can be shown that $\Z_t$ and $\X_t$ share the same marginal law~\citep{liu2022rectified} if $\vec{v}$ is learned well, and therefore simulating this ODE with \( \Z_0 = \X_0 \) results in \( \Z_1 \) being samples from the target distribution \( \pi_1 \). In practice, this ODE can be discretized using the Euler method by selecting \( N \) time steps \( t_0 = 0 < t_1 < \dots < t_N = 1 \), and iteratively updating:
\[
\tilde{\Z}_{t_{k+1}} = \tilde{\Z}_{t_k} + (t_{k+1} - t_k) \vec{v}(\tilde{\Z}_{t_k}, t_k) , \quad \tilde{\Z}_{0} \sim \pi_0.
\]
We use $\tilde{\Z}$ to differentiate the discretized ODE trajectory from its ideal continuous time limit $\Z$. Note the entire process is deterministic once \( \tilde{\Z}_{0} \) is chosen.

%% file: sections/3_method.tex
\begin{algorithm*}[t]
\captionof{algorithm}[srf]{\atphantom\ \ Attention-Modulated Overshoot Sampling for Rectified Flow.}
\begin{spacing}{1.1}
\begin{algorithmic}[1] 

 \AProcedure{OvershootingSampler}{$v, ~\{t_i\}_{i=0}^N, ~ c \in \R^+$}
    
    \AState{{\bf Initialize} $\tilde{\Z}_0 \sim \mathcal{N}\big(\vec{0}, ~\vec{I}\big)$}         
    \AFor{$i \in \{0, \dots, N-1\}$}         
        \AState{Calculate velocity $\vv_i = v(\tilde{\Z}_{t_i}; t_i)$; get the cross attention mask $\mm_i$}          
        \AState{Overshooted ODE update:}  \\   
        \AState{~~~~ ~ $\displaystyle 
        \hat {\Z}_{\vec{o}} = \tilde{\Z}_{t_i} + (\vec{o} - t_i) \circ \vv_i, 
        $~~~with~~~ $\displaystyle 
         \hspace{0pt} \vec{o}= \text{min}\big( t_{i+1} + c (t_{i+1} - t_i) \mm_i, ~1\big). 
         $  
        }\\ 
        \AState{Backward update by adding noise:} \\
    \AState{ $
    \displaystyle \tilde \Z_{t_{i+1}} \gets \vec{a} \hat {\Z}_{\vec{o}} + \vec{b} \xxi_i,$ 
    ~~~~\text{where}~
    $\displaystyle\xxi_i \sim \mathcal{N} \big(\vec{0}, \vec{I}\big),$ and 
        $\displaystyle
        \vec{a} = \frac{s}{\vec{o}}~~~\text{and}~~~\vec{b} = \sqrt{(1 - s)^2 - (\vec{a} (1 - \vec{o}))^2}. 
        $} \\ 
    \EndFor
    \AState{\textbf{return} $\tilde{\Z}_{t_N}$}
  \EndProcedure
\end{algorithmic}
\end{spacing}
\label{alg:srf}
\end{algorithm*}

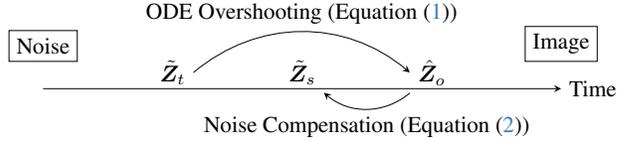
\begin{figure}[t!]
\resizebox{\columnwidth}{!}{ 

\begin{tikzpicture}[>=stealth]

\draw[->] (1,-0.25) -- (9,-0.25) node[right] {Time};

\node[draw, rectangle, align=center, above=-1pt] at (1, 0.2) {Noise};
\node[draw, rectangle, align=center, above=-1pt] at (9, 0.2) {Image};

\node at (3, 0) (Zt) {$\tilde{\Z}_t$};
\node at (5, 0) (Zs) {$\tilde{\Z}_s$};
\node at (7, 0) (Zo) {$\hat{\Z}_o$};

\draw [->, bend left=40] (Zt.east) to node[above] {ODE Overshooting (\Eqref{eq:ode-overshoot})} (Zo.west);

\draw [->, bend left=40] (Zo.south west) to node[below] {Noise Compensation (\Eqref{eq:noise-compensate})} (Zs.south east);

\end{tikzpicture}
}
\vspace{-5pt}
\caption{Visualization of the Overshooting Sampler. 
Given $\tilde{\Z}_t$ at time $t$, we first over-simulate the learned ODE to $\hat{\Z}_o$, and then add noise and return to $\tilde{\Z}_s$. The noise is carefully selected such that $\tilde{\Z}_s$ matches $\X_s$'s marginal distribution. }
\label{fig:overshoot}
\vspace{-5pt}
\end{figure}

\section{Attention Modulated Overshooting Sampler}\label{sec::method}
In this section, we derive the \ourswoatt{} that adds stochastic noise to the Euler sampler while preserving the marginal distribution (Section~\ref{sec::method-overshoot}). Then we illustrate this is equivalent to adding a Langevin dynamics term at each Euler step (Section~\ref{sec::method-langevin}). Importantly, the extra Langevin term can help correct the compounding error from successive Euler steps. When the overshooting strength is high, the stochastic sampler can introduce artifacts. To mitigate this problem, we propose a straightforward attention modulation method that adaptively controls the strength of the overshooting for each image patch based on its attention score with the text content in the prompt (Section~\ref{sec::method-attention-modulation}).

\subsection{Stochastic Sampling via Overshooting}
\label{sec::method-overshoot}
This section provides a derivation of a stochastic sampling method, \ourswoatt{}, for RF-trained models. The main idea is to overshoot the forward Euler step and subsequently compensate with backward noise injection. In the limit of small step sizes, this process converges to a stochastic differential equation (SDE), and we will show rigorously in the subsequent section that the resulting SDE ensures the marginal preserving property according to the Fokker-Planck equation.

Following our notation in the previous section, let $\tilde{\Z}_0 = \X_0 \sim \pi_0$ be a sample from the initial noise distribution, and assume we have obtained $\tilde{\Z}_t$ at time $t$, and want to get a $\tilde{\Z}_{s}$ for the next time point $s = t + \epsilon$, where $\epsilon > 0$ is the step size when denoising. Note that for standard Euler sampler, $\tilde{\Z}_s$ is obtained from $\tilde{\Z}_t + \epsilon \vec{v}(\tilde{\Z}_t, t)$. In comparison, to introduce stochastic noise, we propose the overshooting sampler which consists of the following two steps (See Figure~\ref{fig:overshoot}):

\paragraph{1. ODE Overshooting} First, we temporarily advance $\tilde{\Z}_t$ to $\hat{\Z}_o$ where $o = s + c \epsilon$ (with $c > 0$) denotes an overshooting time point that is larger than $s$. Specifically, we conduct the following the forward Euler step:  
\begin{equation}
\hat{\Z}_o = \tilde{\Z}_t + \vec{v}(\tilde{\Z}_t, t) (o - t) = \tilde{\Z}_t + (1 + c) \epsilon \vec{v}(\tilde{\Z}_t, t). 
\label{eq:ode-overshoot}
\end{equation}
We use $\hat{\Z}_o$ to emphasize that it is reached via overshooting.

\paragraph{2. Noise Compensation} Next, we want to revert from time $o$ to time $s$ by \emph{noising} $\hat{\Z}_o$. Assume we achieve this by computing
\begin{equation}
    \tilde{\Z}_s = a \hat{\Z}_o + b \vec \xi, \quad \vec \xi \sim \mathcal{N}(\vec 0, \mat I).
    \label{eq:noise-compensate}
\end{equation}
Then, the goal is to determine the coefficients $a$ and $b$ here. If we assume the overshooting step is accurate, then $\hat{\Z}_o \myapprox \Z_o \myeq o \X_1 + (1 - o) \X_0 $, where $\myeq$ denotes equality in distribution. Therefore,
\begin{equation}
\begin{split}
\tilde{\Z}_s &\myeq a (o \X_1 + (1 - o) \X_0) + b \xi \\
&= a o \X_1 + \sqrt{a^2 (1-o)^2 + b^2 } \xi',
\end{split}
\end{equation}
where $\xi' \sim \mathcal{N}(\vec 0, \mat I)$ and the last line is derived as both $\xi$ and $\X_0$ are i.i.d Gaussian noises. On the other hand, we know that $\Z_s \myeq s \X_1 + (1 - s) \X_s$. Hence, by matching the coefficients, we get
\begin{equation}
a = \frac{s}{o}, \quad b = \sqrt{(1 - s)^2 - s^2\frac{(1-o)^2}{o^2} }.
\label{eq:overshoot-aandb}
\end{equation}
Recall that $s = t + \epsilon$ and $o = t + (1 + c) \epsilon$, as $\epsilon \rightarrow 0$, the above process (steps 1 and 2) will approach the following limiting stochastic different equation (SDE) (See Appendix~\ref{sec::overshooting_sde_limit} for the derivation):
\begin{equation}
\small
d \Z_t = \bigg( (1 + c) \vec{v}(\Z_t, t) - \frac{c}{t}\Z_t \bigg) dt + \sqrt{\frac{2(1-t)}{t} c}\ d\vec{W}_t, 
\label{eq:overshoot-sde}
\end{equation}
where $\vec{W}_t$ denotes the Brownian motion.

\subsection{Overshooting $\approx$ Euler + Langevin Dynamics} 
\label{sec::method-langevin}
\Eqref{eq:overshoot-sde} can also be derived directly from the Fokker-Planck equation following similar ideas from \cite{song2020score}.  Let $d\Z_t = \vec f_t(\Z_t) dt + \sigma_t d\vec{W}_t$ be a SDE where $\sigma_t\geq 0$ 
is a diffusion coefficient independent of $\X_t$.
(This is true in almost all contemporary diffusion/flow models.)
Denote by $\rho_t$ the density of $\Z_t$. According to the Fokker-Planck equation:
\begin{equation}
\begin{split}
\dot{\rho}_t 
&= - \nabla \cdot \big( \vec f_t \rho_t \big) + \frac{1}{2} \nabla^2 \big( \sigma_t^2 \rho_t \big), \\
&= - \nabla \cdot \big( ( \vec f_t - \frac{\sigma_t^2}{2} \nabla \log \rho_t) \rho_t \big).
\end{split}
\label{eq:fp}
\end{equation}
The last line follows $\nabla {\rho_t}/\rho_t = \nabla \log \rho_t$. Now, let $\vec f_t(\Z_t) = \vec{v}(\Z_t, t) + \frac{\sigma_t^2}{2} \nabla \log \rho_t(\Z_t)$, one can check that for any function $\sigma_t$ (according to \Eqref{eq:fp}), the SDE
\begin{equation}
d \Z_t = \bigg( \vec{v}(\Z_t, t) + \frac{\sigma_t^2}{2} \nabla \log \rho_t(\Z_t) \bigg) dt + \sigma_t d\vec{W}_t,
\label{eq:langevin-sde}
\end{equation}
shares the \emph{same} marginal law at any time $t$ as the ODE
\[
d\Z_t = \vec{v}(\Z_t,t) dt.
\]
This is because the $ \frac {\sigma_t^2}{2} \nabla \log \rho_t$ cancels out in \Eqref{eq:fp} and it reduces to the continuity equation $\dot \rho_t = - \nabla \cdot (\vec v_t\rho_t)$ for the original ODE. Note that
for RF models, $\nabla \log \rho_t(x) = (t \vec{v}(x, t) - x)/(1-t)$. Hence, by choosing $\sigma_t^2 = 2(1-t)c/t$, one exactly recovers \Eqref{eq:overshoot-sde}. The full derivation is provided in Appendix~\ref{sec::langevin_dynamics}.

\paragraph{Langevin Dynamics Corrects Marginals} Note that \Eqref{eq:langevin-sde} can be equivalently viewed 
as the ODE combined with one step of Langevin dynamics: 
\begin{equation}
\small
d \Z_t = \vec{v}(\Z_t, t)dt + \underbrace{\bigg( \alpha_t \nabla \log \rho_t(\Z_t) dt + \sqrt{2 \alpha_t} d\vec{W}_t \bigg)}_{\text{Langevin Dynamics}},
\label{eq:langevin-sde-2}
\end{equation}
where $\alpha_t = \sigma_t^2/2$ is the step size for the Langevin dynamics. As $\tilde \Z_t$ comes from successive Euler steps, it does not necessarily follow $\rho_t$ of $\Z_t$.
In this case, the Langevin term in \Eqref{eq:langevin-sde-2} helps move $\tilde{\Z}_t$ towards the desired marginal $\rho_t$~(see e.g., \cite{karras2022elucidating}). 
In Figure~\ref{fig:toy}, we visualize this correction effect on a 2D toy problem.

\begin{figure}[h!]
    \centering
    \vspace{-5pt}
    \includegraphics[width=1.0\columnwidth]{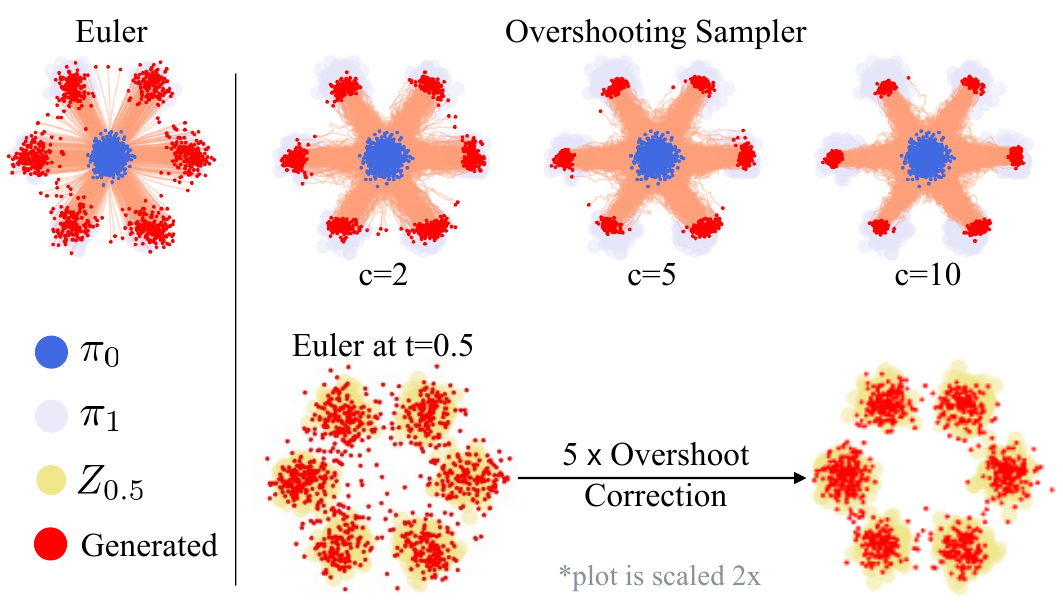}
    \caption{\textbf{Euler versus Overshooting on a toy dataset.}
     The noise ($\pi_0$) and data ($\pi_1$) distributions are shown as blue and light-purple dots. \textbf{Top:} The samples from Euler deviate from $\pi_1$. Overshooting sampler helps correct the marginal. As $c$ increases, the correction effect is stronger, but it also introduces smoothing artifacts. \textbf{Bottom:} Starting with $\tilde \Z_t$ ($t=0.5$) from the Euler sampler, if we apply 5 times of (Overshooting - Euler), i.e., the Langevin dynamics part in \Eqref{eq:langevin-sde-2}), the samples align better with $\Z_{0.5}$.}
    \label{fig:toy}
    \vspace{-5pt}
\end{figure}

\paragraph{Remark} It is worth noting that the above equivalence (between \Eqref{eq:overshoot-sde} and Overshooting) only holds in the limit of infinitely small step sizes. Compared to applying the Euler discretization of the SDE in \Eqref{eq:overshoot-sde}, we empirically found that the Overshooting sampler tends to be more stable and yields better text rendering on real images (See Table~\ref{tab:ablation_sde}).

\begin{table}[h!]
\centering
\resizebox{\columnwidth}{!}{
\begin{tabular}{l ccc}
\toprule
& \textbf{20 Steps} & \textbf{50 Steps} & \textbf{100 Steps} \\
\midrule
Discretize SDE in \Eqref{eq:overshoot-sde} & 51.5~\% & 71.0~\% & 76.0~\%\\  

Overshooting & \textbf{68.5~\%} & \textbf{81.5~\%} & \textbf{81.5~\%} \\  

\bottomrule
\end{tabular}
}
\vspace{-5pt}
\caption{\textbf{Text Rendering Accuracy between Discretize SDE in \Eqref{eq:overshoot-sde} and Overshooting}. We conduct a human study on the text rendering accuracy and compare discretizing SDE in \Eqref{eq:overshoot-sde} and Overshooting. We compare the two samplers for 20, 50, and 100 inference steps, and found that with fewer steps, the Overshooting sampler demonstrates a more significant improvement over applying the discretized SDE. This is because when $\epsilon$ (the step size) is large, the Euler discretization becomes inaccurate and unstable. }
\label{tab:ablation_sde}
\vspace{-5pt}
\end{table}

\subsection{Attention Modulation}
\label{sec::method-attention-modulation}
In practice, while increasing $c$ can enhance text rendering quality, it may also introduce artifacts (see Figure~\ref{fig:details} for examples). This is because, with larger values of $c$, the single-step Langevin correction in \Eqref{eq:langevin-sde-2} becomes less accurate. To address this issue, we propose dynamically adjusting the overshooting strength for different image patches based on their attention scores to the text content in the prompt. Simply put, this approach increases overshooting for areas related to the text while applying less to the rest of the image.

More concretely, assume the image consists of $h \times w$ patches, where $h$ and $w$ denote the height and width dimensions of the 2D image tokens (e.g., 64 $\times$ 64 in our experiment). Let $\vec x^{\texttt{image}}_{h,w}$ be the $(h, w)$-th image patch token. Let $\{\vec x^{\texttt{text}}_i\}_{i=1}^n$ denote the set of tokens within the language instruction prompt for generating text and $n$ is the total number of text-related tokens (e.g., so $\{\vec x^{\texttt{text}}\}$ can be the `Tokyo Halloween Night' in the prompt ``A poster design with a title text of `Tokyo Halloween Night'). Then, we construct a mask $\vec{m} \in \mathbb{R}^{h \times w}$ in the following way:
\begin{equation} 
\vec{m}_{h, w} = \sum_{i=1}^n \frac{\exp(Q(\vec  x^{\texttt{text}}_i)^\top K(\vec x^{\texttt{image}}_{h,w}))}{\sum_{h', w'}  \exp(Q(\vec  x^{\texttt{text}}_i)^\top K(\vec x^\texttt{image}_{h',w'})}, 
\end{equation}
where $Q$ and $K$ denote the query and key vectors in attention. We then average the attention map over different layers and heads and rescale its values between 0 and 1. After that, we apply the resulting attention map, $\vec{m}$, to give different image patches different amounts of overshooting. Specifically, assume $\vec{o} \in [0, 1]^{h \times w}$, where $\vec{o}_{h, w} = s +  c \epsilon \vec{m}_{h, w}$, we have
\begin{equation}
\small
\hat{\Z}_o = \tilde{\Z}_t + \vec{v}(\tilde{\Z}_t, t) \circ (\vec{o} - t) = \tilde{\Z}_t + \epsilon (1 + c \vec{m}) \circ \vec{v}(\tilde{\Z}_t, t),
\end{equation}
where $\circ$ denotes element-wise product. 
The noise compensation step in \Eqref{eq:overshoot-aandb} is similarly adapted, where $a, b, o$ are replaced by other vector counterparts:
\begin{equation} 
\vec{a} = \frac{s}{\vec{o}}, \quad \vec{b} = \sqrt{(1 - s)^2 - s^2 \frac{(1 - \vec{o})^2}{\vec{o}^2}}.
\end{equation}
In the above, all operations are elementwise. Both $\vec{a}$ and $\vec{b}$ have dimensions $\mathbb{R}^{h \times w}$. The entire \ours{} sampling process is provided in  Algorithm~\ref{alg:srf}.

%% file: sections/4_related_work.tex
\section{Related Work} This section gives an overview of diffusion and flow-based generative models, followed by a discussion on deterministic and stochastic sampling techniques, and recent advances in enhancing text rendering in text-to-image generation.

\paragraph{Diffusion and Flow Models.} Diffusion models~\cite{song2019generative, song2020score, ho2020denoising} have emerged as powerful generative frameworks capable of producing high-fidelity data, including images, videos, audio, and point clouds~\cite{dhariwal2021diffusion, saharia2022photorealistic, ho2022video, podell2023sdxl, betker2023improving, baldridge2024imagen}. These models add noise to data in a forward process, then learn to reverse this noise to generate new samples, thereby modeling the data distribution through a progressive denoising process. Recently, Rectified Flow (RF)~\cite{liu2022flow, lipman2022flow, albergo2022building, heitz2023iterative}—also known as Flow Matching, InterFlow, and IADB—has been proposed as a novel approach that leverages an ordinary differential equation (ODE) with deterministic sampling during inference. RF simplifies the diffusion and denoising process, offering computational efficiency while maintaining high-quality generation, and positioning itself as a compelling alternative to traditional diffusion models. The rectified flow model has proven successful in various applications, including image generation\cite{esser2024scaling}, sound generation~\cite{fei2024flux, le2024voicebox} and video generation~\cite{polyak2024movie}.

\paragraph{Deterministic and Stochastic Sampling Methods.} Sampling strategies in generative modeling are crucial as they influence the quality, diversity, and efficiency of generated samples. Deterministic sampling methods, such as those based on ODE solvers~\cite{song2020score, song2020denoising, lu2022dpm, karras2022elucidating}, provide computational efficiency and stability but may lead to poorer output quality~\cite{karras2022elucidating}. On the other hand, stochastic sampling methods introduce randomness into the sampling process, offering an alternative approach with added variability in intermediate steps. \citet{meng2023distillation} introduced an N-step stochastic sampling method for distilled diffusion models, where noise is added at intermediate steps to achieve efficient sampling with as few as 2-4 steps. This approach provides an alternative to deterministic sampling, allowing the model to produce high-quality samples. \citet{karras2022elucidating} proposed a hybrid stochastic sampling technique that combines deterministic ODE steps with noise injection. In their method, noise is temporarily added at each step to improve sampling quality, followed by an ODE backward step to maintain the correct distribution. This hybrid approach results in better output quality compared to purely deterministic sampling methods. However, these existing methods are specifically designed for diffusion models. In contrast, we propose a stochastic sampler for rectified flow, providing an alternative solution to the traditional Euler method.

\paragraph{Enhancing Text Rendering in T2I Generation.} Accurate text rendering in text-to-image (T2I) generation models remains a significant challenge, as models often struggle to produce text within images that precisely matches the prompts, leading to incoherent or incorrect textual content. Classifier-Free Guidance (CFG)\cite{ho2022classifier} can alleviate this issue by adjusting the influence of the text prompt during sampling, effectively balancing prompt guidance and the diversity of the generated content. Scaling or enlarging the text encoder has been shown to benefit text rendering\cite{podell2023sdxl}. Additionally, using a T5 text encoder significantly improves text rendering performance~\cite{saharia2022photorealistic, esser2024scaling}. Specialized fine-tuning approaches have also been explored, where pretrained text-to-image models are adapted with architectures designed specifically for text rendering tasks~\cite{ramesh2022hierarchical, liu2022character, ma2023glyphdraw, chen2024textdiffuser, tuo2023anytext, chen2023textdiffuser, zhao2023udifftext, liu2024glyph}. These methods enhance the model's ability to generate accurate textual content but typically require extensive retraining or fine-tuning, which can be computationally intensive. In contrast, our work introduces a training-free approach that enhances text rendering during inference. By incorporating stochastic sampling and an attention mechanism into the Rectified Flow framework, we improve text rendering quality without modifying the underlying model or incurring additional training costs.

%% file: sections/5_exp.tex
\section{Experiment}

We conduct experiments with several open-source text-to-image models based on Rectified Flow, including Stable Diffusion 3 (medium)~\cite{esser2024scaling}, Flux (dev)~\cite{blackforestlabs_flux_2023}, and AuraFlow~\cite{huggingface_auraflow_2023}. Image generation was performed using the NVIDIA A40 GPU during inference. To ensure high-quality visual assessment, all output images were generated at a resolution of 1024x1024 pixels. Detailed model configurations and hyperparameter settings can be found in the Appendix~\ref{sec::experiment_setting}.

\begin{figure*}[ht!]
    \centering
    \includegraphics[width=1.0\textwidth]{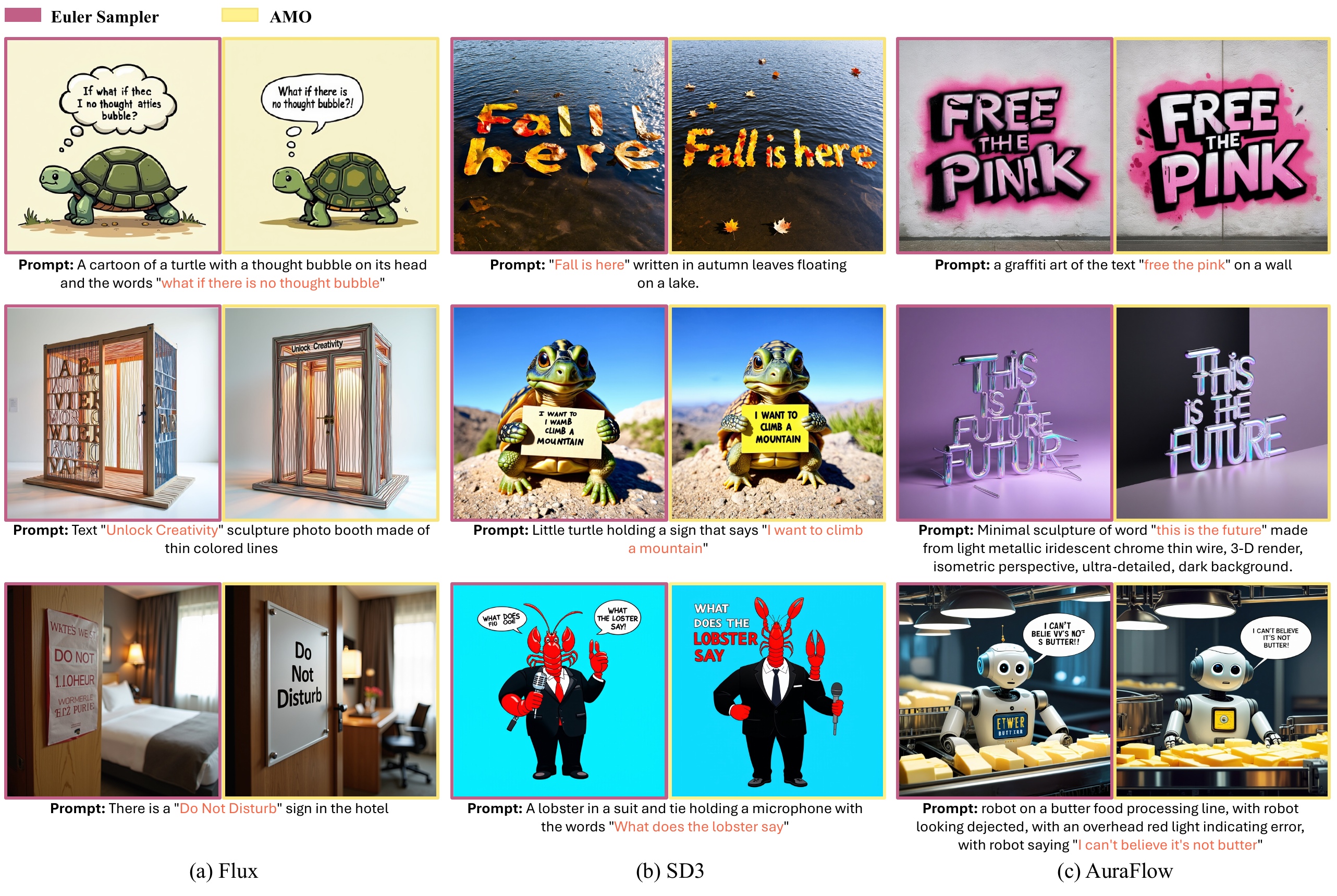}
    \vspace{-2\baselineskip}
    \caption{\textbf{Comparison of text rendering quality between Euler and our stochastic sampling method} across three different text-to-image models: (a) Flux, (b) Stable Diffusion 3 (SD3), and (c) AuraFlow. All results are generated using the same random seed for consistent comparison. Within each pair of images, the left column corresponds to the Euler sampler, while the right column displays the results from our method.  Our approach consistently generates clearer and more legible text that closely matches the provided prompts. Additional examples are provided in the Appendix.}
    \label{fig:euler_ours}
    \vspace{-5pt}
\end{figure*}

\begin{table}[ht!]
\centering
\resizebox{\columnwidth}{!}{
    \setlength{\tabcolsep}{0.3cm}
    \begin{tabular}{lcccccccc}
    \toprule
     & \textbf{FID~$(\downarrow)$} & \textbf{CLIP~$(\uparrow)$} & \textbf{OCR-A~$(\uparrow)$} & \textbf{OCR-F~$(\uparrow)$} & \textbf{CR~$(\uparrow)$} \\
    \midrule
    SD3 (Euler) & 81.1 & \textbf{0.319} & 0.256 & 0.473 & 32.5 \% \\
    SD3 (AMO) & \textbf{80.9} & 0.317 & \textbf{0.279} & \textbf{0.482}  & \textbf{43.0 \%} \\
    \midrule
    Flux (Euler) & 111.8 & 0.299 & 0.313 & 0.458 & 74.0 \%\\ 
    Flux (AMO) & \textbf{108.9} & \textbf{0.303} & \textbf{0.381} & \textbf{0.494} & \textbf{82.5 \%}\\
    \midrule
    AuraFlow (Euler) & 128.4 & \textbf{0.299} & 0.075 & 0.238  & 1.0 \%\\
    AuraFlow (AMO) & \textbf{117.5} & 0.298 & \textbf{0.082} & \textbf{0.258}  & \textbf{3.0 \%}\\
    \bottomrule
    \end{tabular}
    }
    \caption{Quantitative evaluation of AMO against the Euler sampler. \textbf{OCR-A} and \textbf{OCR-F} are short for \textbf{OCR (Accuracy)} and \textbf{OCR (F-Measure)}. \textbf{CR} is the correction rate from human evaluation.}
    \label{tab:main-results}
    \vspace{-5pt}
\end{table}

\paragraph{Evaluation Metrics.}
We use a combination of automated and human evaluations to assess the performance of our models. For automated evaluation, we adopt benchmarks including DrawTextCreative~\cite{liu2022character}, ChineseDrawText~\cite{ma2023glyphdraw} and TMDBEval500~\cite{chen2024textdiffuser}, which comprises a total of 893 prompts drawn from various data sources. To assess the correctness of rendered text, we compute \textbf{OCR Accuracy} and \textbf{OCR F-measure} using a pre-trained Mask TextSpotter v3 model \cite{liao2020mask}. We evaluate the samples' visual-textual alignment using \textbf{CLIP Score}, specifically CLIP L/14~\cite{radford2021learning}, and also compute \textbf{FID} for overall visual quality between the CLIP image features and validation set images. 

However, automated OCR tools, such as those in MARIO-Eval~\cite{chen2023textdiffuser}, showed limitations in accurately detecting text from the generated images. This is partly because general-purpose text-to-image models can produce diverse and artistic fonts that humans can readily understand, but OCR models cannot recognize accurately. It is also worth noting that OCR accuracy can be negatively impacted by extraneous content in images, such as posters, hurting recalls but not affecting the overall human perception. This discrepancy is especially pronounced for general-purpose text-to-image models (see Appendix~\ref{sec:evaluation_issue}). To address these limitations, we conduct \textbf{human evaluation} by manually assessing the correctness of text rendering on samples covering 100 prompts. Further details are provided in the Appendix~\ref{sec::experiment_setting}.

\subsection{Comparison with Euler Sampler}

In this section, we compare \ours{} against the Euler sampler both quantitatively and qualitatively. We found that \ours{} significantly outperforms Euler on text rendering without sacrificing overall visual quality.

\paragraph{Quantitative Results}
As shown in Table~\ref{tab:main-results}, \ours{} achieves 82.5\% accuracy on Flux model in human evaluation, notably surpassing the 74\% accuracy of the standard Euler sampler. In addition, \ours{} improves the OCR metrics across all three text-to-image models, demonstrating a substantial enhancement in the model's ability to render text accurately. Furthermore, compared to the standard Euler method, \ours{} yields better FID scores, indicating superior overall image quality. 
As shown in Figure~\ref{fig:euler-ours-number}, we evaluated the performance of \ours{} using 20, 50, and 100 steps. Our results demonstrate that \ours{} consistently outperforms the deterministic sampler across all step counts, with a better performance improvement in the low-step scenarios. 

\begin{figure}[ht!]
    \centering
    \vspace{-5pt}
    \includegraphics[width=\columnwidth]{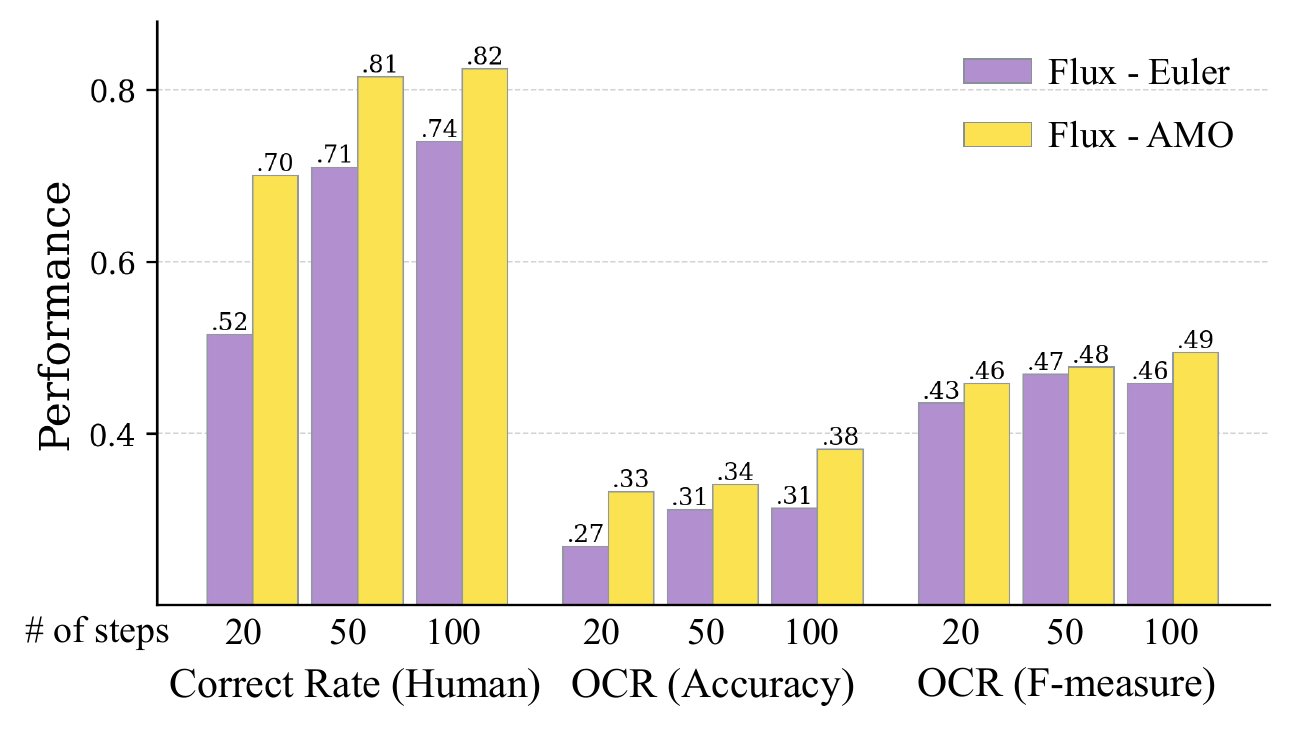}
    \vspace{-5pt}
    \caption{The comparison of Euler sampler and \ours{} across different sampling steps (20, 50, and 100 steps). \ours{} consistently outperforms the deterministic sampler on text rendering performance across all step sizes.}
    \label{fig:euler-ours-number}
    \vspace{-5pt}
\end{figure}

\paragraph{Qualitative Results.} Figure~\ref{fig:euler_ours} presents a visual comparison between \ours{} and the standard Euler sampling applied to Flux, Stable Diffusion 3, and AuraFlow. Images generated by \ours{} exhibit clear and legible text that closely aligns with the given prompts. In contrast, the Euler method frequently produces misaligned and misspelled text.

\subsection{Ablation Studies}
In this section, we conduct a detailed ablation study to analyze the impact of different components in \ours{} on text rendering accuracy and image quality using the Flux model.

\paragraph{Impact of Components in \ours{}.} 
The \ours{} sampler essentially consists of three parts: ODE overshooting (\textbf{O}), noise compensation (\textbf{N}), and attention modulation (\textbf{A}). We ablate their individual contribution, by comparing \ours{} without all of them (\xmark, ~\xmark, ~\xmark) (i.e., the Euler sampler), with only overshooting (\checkmark, ~\xmark, ~\xmark), without attention modulation (\checkmark, \checkmark, ~\xmark) and the full \ours{}. Results are summarized in Table~\ref{tab:ablation}. It is observed that \emph{both} overshooting and noise compensation are crucial for achieving accurate text rendering and high image quality. Notably, introducing overshooting alone results in a 0\% correct rate, as the marginal law is not preserved. The full \ours{} results in a similar performance to the Overshooting sampler (\ours{} without attention modulation), we think this is because metrics like OCR-F do not capture the image generation quality well. Therefore, we provide a visualization of samples from the Overshooting sampler against those from \ours{} in Figure~\ref{fig:details}. We observe that Overshooting without attention modulation can result in an over-smoothing effect, making the generated samples lose high-frequency details (See the parrot feather and smoke). This confirms the necessity of attention modulation.

\begin{table}[t!]
\centering
\resizebox{\columnwidth}{!}{
\begin{tabular}{l ccccc}
\toprule
\textbf{(O, N, A)} & \textbf{FID~$(\downarrow)$} & \textbf{CLIP~$(\uparrow)$} & \textbf{OCR-A~$(\uparrow)$} & \textbf{OCR-F~$(\uparrow)$} & \textbf{CR~$(\uparrow)$} \\
\midrule
(\xmark, ~\xmark, ~\xmark) & 111.8 & 0.299 & 0.313 & 0.458 & 74.0~\% \\  

(\checkmark, ~\xmark, ~\xmark) & 367.8 & 0.126 & 0.030 & 0.000 & 0.0~\%\\  


(\checkmark, \checkmark, ~\xmark) & 109.5 & \textbf{0.304} & 0.368 & \textbf{0.503} & 81.5~\%\\

(\checkmark, \checkmark, \checkmark)  & \textbf{109.0} & 0.301 & \textbf{0.381} & 0.494 & \textbf{82.5}~\% \\

\bottomrule
\end{tabular}
}
\caption{{Ablation of each component in AMO on Flux.}}
\label{tab:ablation}
\vspace{-5pt}
\end{table}

\begin{figure}[ht!]
    \centering
    \vspace{-5pt}
    \includegraphics[width=\columnwidth]{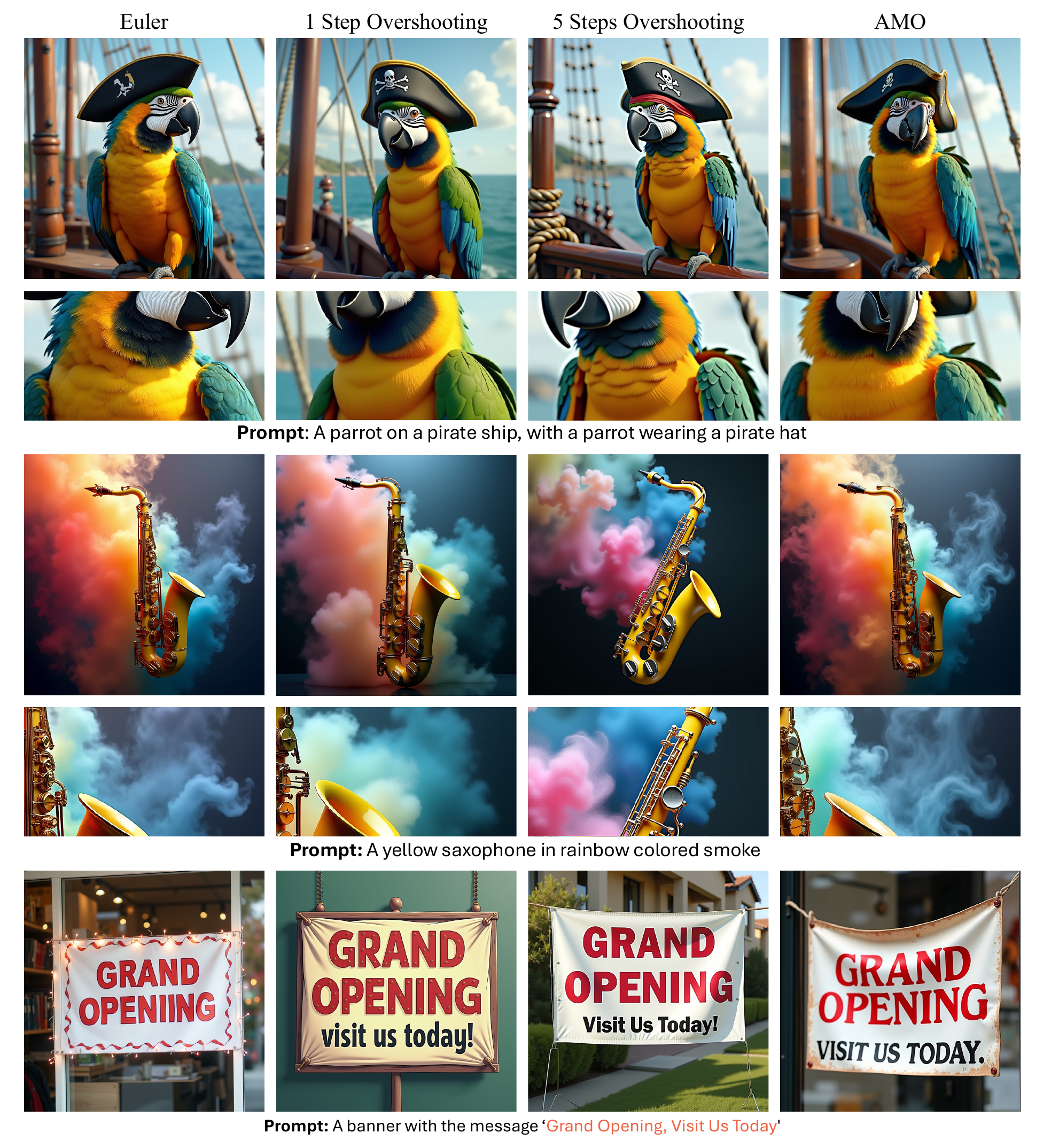}
    \caption{\textbf{Image Quality for Euler, Overshooting, and \oursbf{}.} Please zoom in for details. \textbf{Bottom:} both Overshooting (\ours{} without attention modulation) and \ours{} render the correct texts, while Euler renders misspelled texts. \textbf{Top:} Looking at the parrot's feather or the smoke behind the saxophone, Euler generates high-fidelity high-frequency details while the Overshooting sampler over-smooths the image (fewer details). \ours{} preserves the details from the Euler, with attention modulation. In addition, we conduct 5 Steps Overshooting, meaning that we use $c' = c/5$ but apply (Overshoot - Euler) 5 times (i.e., the Langevin step in \Eqref{eq:langevin-sde-2}) followed by 1 Euler step in the end at each time $t$. We see that with smaller $c$ but more local Langevin steps the smoothing effect also goes away, but in practice, this requires more model evaluations.
    }
    \label{fig:details}
    \vspace{-5pt}
\end{figure}

\paragraph{Impact of Overshooting Strength.}
We further examine the effect of $c$, which governs the maximum overshooting strength per step. Results are shown in Figure~\ref{fig:step_and_perf}. Generally, increasing $c$ enhances text rendering accuracy, with performance plateauing at $c \geq 2$ and occasionally declining for very large values of $c$. Consequently, we set $c = 2$ as the default in practice. To illustrate the degradation in image quality caused by high $c$, we direct readers to Appendix~\ref{sec::large_c_artifact} for examples.

\begin{figure}[ht!]
    \vspace{-5pt}
    \centering
    \includegraphics[width=\columnwidth]{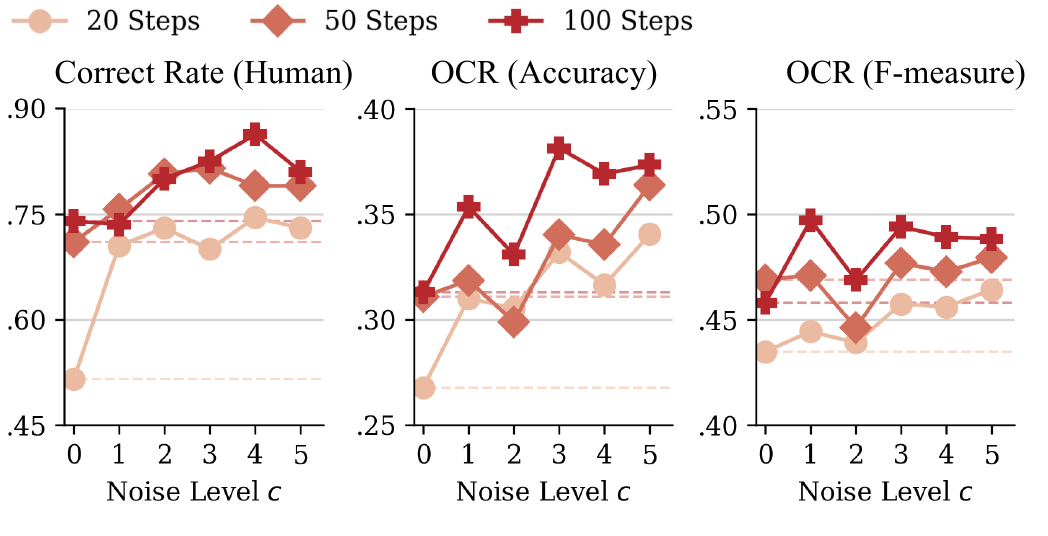}
    \vspace{-5pt}
    \caption{\textbf{Impact of overshooting strength on Text Rendering Performance.} This figure illustrates how varying the overshooting strength parameter $c$ in \ours{} affects the Flux model's text rendering performance. Larger $c$ tends to achieve higher text rendering quality. We observe that $c \geq 2$ is usually good and use $c=2$ as the default value.}
    \label{fig:step_and_perf}
    \vspace{-5pt}
\end{figure}

\subsection{Comparison with Finetuned T2I models}
In this section, we evaluate both text rendering capability and overall quality for two image generation model families for text rendering: \textbf{1)} General-purpose T2I models trained on extensive image datasets using rectified flow, such as Stable Diffusion 3, Flux and AuraFlow; and \textbf{2)} Task-specific T2I models explicitly trained on datasets of ground-truth written text, such as GlyphControl~\cite{liu2024glyph} and TextDiffuser~\cite{chen2024textdiffuser}.

\begin{figure}[ht!]
\vspace{-5pt}
    \centering
    \includegraphics[width=1.0\columnwidth]{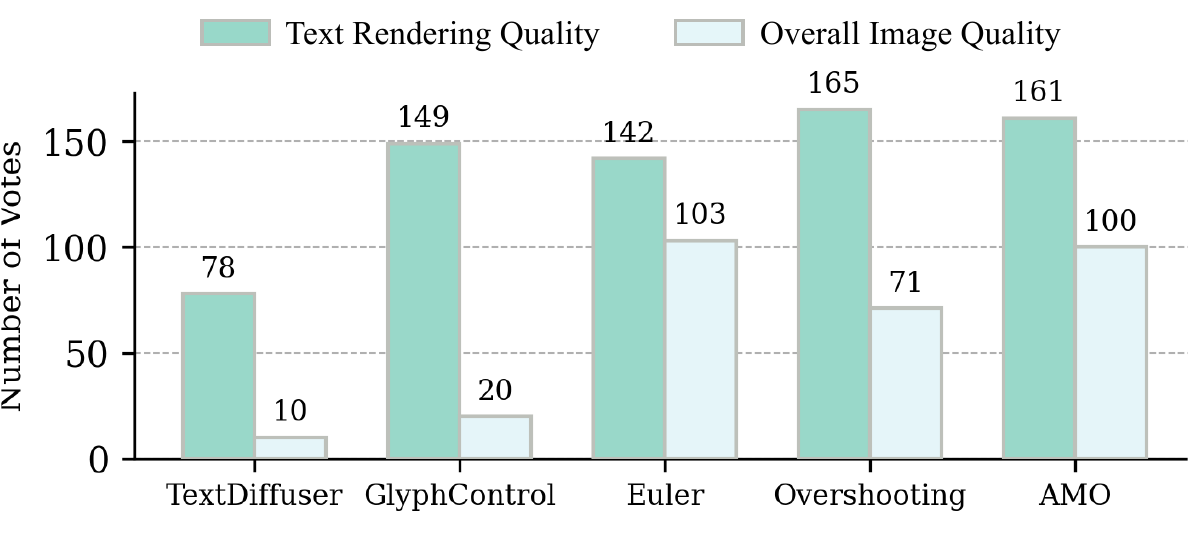}
    \vspace{-5pt}
    \caption{\textbf{Results of human evaluation comparing text rendering quality and overall image quality across five methods.} Participants viewed five images, each generated by one of the methods, and were asked to vote for: (1) the models with the best text rendering quality (multiple-choice), and (2) the model with the best overall image quality (single-choice). The chart shows the number of votes received by each method for both criteria. }
    \label{fig:human_study}
    \vspace{-5pt}
\end{figure}

Specifically, we conducted a human evaluation study to assess image generation and text rendering quality across different methods. The study involved 304 cases (selected randomly from DrawTextCreative, ChineseDrawText, and TMDBEval500) and 15 participants, with each case presenting two questions: (1) "Which of the following images exhibits the highest text rendering quality? (multiple-choice)" and (2) "Which of the following images demonstrates the best overall image quality? (single-choice)". 

The results are summarized in Figure~\ref{fig:human_study}. As shown, \ours{} achieves text rendering accuracy on par with the GraphControl model, which is specifically trained for this task, while delivering superior overall image quality due to the advanced capabilities of the Flux model. Crucially, since \ours{} is a training-free method—unlike approaches such as TextDiffuser or GlyphControl—it can be effortlessly applied to any existing model without requiring further training or risking potential image quality loss from fine-tuning. Additionally, when compared to Overshooting (AMO without attention modulation), \ours{} demonstrates a clear advantage: while both yield similar text rendering quality, the Overshooting sampler diminishes image quality relative to the Euler sampler, whereas \ours{} maintains parity with Euler. This underscores the importance of attention modulation for optimal performance. For additional qualitative comparisons, please refer to Appendix~\ref{sec::more_visual_results}.

%% file: sections/6_conclusion.tex
\vspace{10pt}
\section{Conclusion and Limitation}
Accurate text rendering has been a critical challenge  in text-to-image generation. 
Existing solutions, such as specialized fine-tuning or scaling up the text encoder, often require modifications to the training process, which can be computationally expensive and time-consuming.
This work introduces a training-free method, the \ourswoatt{} sampler, that enhances text rendering by strategically incorporating curated randomness into the sampling process.  Importantly, \ourswoatt{} significantly improves text rendering accuracy with almost no overhead compared to the vanilla Euler sampler. 
In particular, \ourswoatt{} alternates between overshooting the learned ODE and re-introducing noise, while ensuring the marginal laws are well-preserved. 
In addition, we introduce an attention modulation that quantitatively controls the degree of overshooting, concentrating on text regions within the image while minimizing interference with other areas. We validate \ours{} on popular open-source text-to-image flow models, including Stable Diffusion 3, Flux, and AuraFlow. Results demonstrate that \ours{} consistently outperforms baseline methods in text rendering accuracy without degrading the image quality of the pretrained model. 

One limitation of this work is the lack of systematic evaluation of overshooting's impact on specific aesthetics or its applicability beyond image generation. Future work could explore extending \ours{} to other domains.

%% file: sections/7_acknowledgement.tex
\section{Acknowledgment}

We thank Yeqing Li, Eugene Ie and Zihui Xue for their constructive feedback, and the anonymous reviewers for their helpful suggestions.
This work was supported in part by the National Science Foundation (NSF) under Grant NSF CAREER 1846421, Office of Navy Research, NSF AI Institute for Foundations of Machine Learning (IFML) and a grant from Google. 

%% file: sections/7_appendix.tex
\onecolumn
\section{Appendix}

\subsection{The SDE Limit of the Overshooting Sampler}\label{sec::overshooting_sde_limit}

In this section, we derive the asymptotic limit of the overshooting sampler's update as a stochastic differential equation (SDE) by considering the infinitesimal step size $\epsilon \to 0$ in the definitions of $s$ and $o$.
Recall that \begin{equation}\label{eq
} s = t + \epsilon, \quad o = s + c \epsilon = t + (1 + c)\epsilon, \end{equation} where $c$ is a constant parameter.
Combining the update equations (see \Eqref{eq:ode-overshoot} and \Eqref{eq:noise-compensate}), we obtain 
\begin{equation}
\begin{split} 
\tilde{\Z}_s &= a \hat{\Z}_o + b \vec{\xi} \\
&= a \tilde{\Z}_t + a (o - t) \vec{v}(\tilde{\Z}_t, t) + b\vec{\xi} \\
&= \tilde{\Z}_t + \underbrace{(a - 1) \tilde{\Z}_t + a (o - t) \vec{v}(\tilde{\Z}_t, t)}_{\text{Drift}} + \underbrace{b \vec{\xi}}_{\text{Diffusion}}, 
\end{split} 
\end{equation} 
We aim to express the update in the form \begin{equation}\label{eq
} \tilde{\Z}_{t+\epsilon} \approx \tilde{\Z}_t + \vec{v}^{\text{adj}}(\tilde{\Z}_t, t) \epsilon + \sigma_t \sqrt{\epsilon} \vec{\xi}_t, \end{equation} which corresponds to the Euler–Maruyama discretization of the SDE \begin{equation}\label{eq
} d \Z_t = \vec{v}^{\text{adj}}(\Z_t, t) dt + \sigma_t d\vec W_t, \end{equation} with $\vec W_t$ denoting a standard Wiener process.

To this end, we perform a first-order Taylor expansion assuming $\epsilon \to 0$: First, we compute $a - 1$: 
\begin{equation}
\begin{split} 
a - 1 &= \frac{s}{o} - 1 = \frac{s - o}{o} = \frac{-c \epsilon}{o} \approx -\frac{c \epsilon}{t}, 
\end{split} 
\end{equation} 
where we use the approximation $o \approx t$ for small $\epsilon$.
Next, we compute $a (o - t)$: 
\begin{equation}
\begin{split} 
a (o - t) &= \frac{s}{o} (o - t) = \frac{t + \epsilon}{t + (1 + c) \epsilon} (1 + c) \epsilon \approx (1 + c) \epsilon, 
\end{split} 
\end{equation}
Now, we compute $b^2$: 
\begin{equation}
\begin{split} 
b^2 &= (1 - s)^2 - s^2 \left( \frac{1 - o}{o} \right)^2 \\ 
&= s^2 \left( \left( \frac{1 - s}{s} \right)^2 - \left( \frac{1 - o}{o} \right)^2 \right) \\
&= s^2 \left( f(s) - f(o) \right), 
\end{split} 
\end{equation} 
where $f(x) = \left( \frac{1 - x}{x} \right)^2$. Using a first-order Taylor expansion of $f(x)$ around $x = s$, we have 
\begin{equation}
\begin{split} 
f(o) &\approx f(s) + f'(s) (o - s) \\
&= f(s) + f'(s) c \epsilon \\
&= 2 \frac{1 - t}{t} c \epsilon , 
\end{split}
\end{equation} 
Combining the above results, the update equation becomes \begin{equation}
\tilde{\Z}_{t+\epsilon} \approx \tilde{\Z}_t + \vec{v}^{\text{adj}}(\tilde{\Z}_t, t) \epsilon + \sigma_t \sqrt{\epsilon} \vec{\xi}_t, 
\end{equation} 
where the adjusted velocity is 
\begin{equation}
\begin{split} 
\vec{v}^{\text{adj}}(\tilde{\Z}_t, t) &= \left( \frac{a - 1}{\epsilon} \right) \tilde{\Z}_t + \left( \frac{a (o - t)}{\epsilon} \right) \vec{v}(\tilde{\Z}_t, t) \\
&= -\frac{c}{t} \tilde{\Z}_t + (1 + c) \vec{v}(\tilde{\Z}_t, t), 
\end{split} 
\end{equation} 
Thus, the limit SDE is 
\begin{equation}
d \Z_t = \vec{v}^{\text{adj}}(\Z_t, t) dt + \sigma_t d\vec W_t = \left( (1 + c) \vec{v}(\Z_t, t) - \frac{c}{t} \Z_t \right) dt + \sqrt{\frac{2 c (1 - t)}{t}} d\vec W_t. 
\end{equation}
This provides the SDE limit of the overshooting sampler as $\epsilon \to 0$.

\subsection{Stochastic Sampler by Fokker Planck Equation}\label{sec::langevin_dynamics}
As mentioned in Section~\ref{sec::method-langevin}, according to the Fokker-Planck Equation, for an ODE $d\Z_t = \vec{v}(\Z_t, t) dt$, we can construct a family of SDEs that share the same marginal law as the ODE at all $t$:
$$
d\Z_t = \bigg( \vec{v}(\Z_t, t) + \frac{\sigma_t^2}{2} \nabla \log \rho_t(\Z_t) \bigg) dt + \sigma_t d\vec W_t.
$$

Now, we only need to figure out $\log \rho_t(\Z_t)$ and then we can find the corresponding $\sigma_t^2$ that matches with the limiting SDE of the overshooting algorithm. To this end, we present the next two lemmas before presenting the equivalence.

\begin{lemma}
Assume random variables $\X = \vec{Y} + \Z$, where $\vec{Y}$ and $\Z$ are independent, then 
$$
\nabla_x \log \rho_\X(x) = \mathbb{E}[ \nabla_{\vec{y}} \log \rho_{\vec{Y}}(\vec{Y}) \mid \X = x] =  \mathbb{E}[ \nabla_z \log \rho_\Z(\Z) \mid \X = x],
$$
where $\rho_\Z$ and $\rho_{\vec{Y}}$ are the density functions of $\Z$ and $\vec{Y}$, respectively. 
\label{lem:independent-score}
\end{lemma}
\begin{proof}
\begin{align*}
\nabla_x \log \rho_\X(x) &= \frac{\nabla_x \rho_\X(x)}{\rho_\X(x)} \\
&= \frac{\nabla_x \int_z \rho_{\X, \Z}(x, z) dz}{\rho_\X(x)} \\
&= \frac{\int_z \rho_\Z(z) \nabla_x \rho_{\vec{Y}}(x - z) dz}{\rho_\X(x)} \qquad\text{\textcolor{magenta}{// $\vec{Y}$ and $\Z$ are independent}}\\
&= \int_z \frac{\nabla_x \rho_{\vec{Y}}(x - z)}{\rho_{\vec{Y}}(x - z)} \frac{\rho_\Z(z) \rho_{\vec{Y}}(x - z)}{\rho_\X(x)} dz  \\
&= \int_z \nabla_x \log \rho_{\vec{Y}}(x - z) \frac{\rho_\Z(z) \rho_{\vec{Y}}(x - z)}{\rho_\X(x)} dz \\
&= \mathbb{E} [ \nabla_x \log \rho_{\vec{Y}}(\X- \Z) \mid \X = x] \\
& = \mathbb{E} [ \nabla_{\vec{y}} \log \rho_{\vec{Y}}(\vec{Y}) \mid \X = x].
\end{align*}
\end{proof}

\begin{lemma}
Given the linear interpolation in Rectified Flow $\X_t = t\X_1 + (1-t) \X_0$, where $\X_0 \sim \mathcal{N}(0, I)$, we have
\begin{align} \label{eq:nabla-log-rhot}
\nabla_x \log \rho_t(x) = \frac{t\vec{v}(x, t) - x}{1 - t}.
\end{align}
\end{lemma}
\begin{proof}
As $\X_0$ and $\X_1$ are independent since $\X_0$ is the standard multivariant Gaussian and $\X_1$ is the data distribution, take $\vec{Y} = t \X_1$ and $\Z = (1-t)\X_0$. According to Lemma~\ref{lem:independent-score},  we have
\begin{align*} 
\begin{split}
\nabla_x \log \rho_t(x) 
& = \mathbb{E} [ \nabla_z \log \rho_\Z(\Z) \mid \X_t = x] \\
&= - \frac{1}{(1 - t)^2} \mathbb{E}[\Z \mid \X_t = x ]  \qquad\text{\textcolor{magenta}{// $\Z \sim \mathcal{N}(0, (1-t)^2 I)$}}\\
&= -\frac{1}{1 - t} \mathbb{E}[\X_0 \mid \X_t = x] \qquad\text{\textcolor{magenta}{// $\Z = (1-t)\X_0$}}\\
&= \frac{1}{1-t} \mathbb{E}[ t (\X_1 - \X_0) - \X_t \mid \X_t = x] \qquad\text{\textcolor{magenta}{// $\X_t = t\X_1 + (1-t) \X_0$}}\\
&= \frac{t\vec{v}(x, t) - x}{1-t} \qquad\text{\textcolor{magenta}{// $E[\X_1 - \X_0 \mid \X_t] = \vec{v}(\X_t, t)$}}.
\end{split}
\end{align*} 
\end{proof}
Plugging in \Eqref{eq:nabla-log-rhot} to the SDE, we have
$$
d\Z_t = \bigg( \vec{v}(\Z_t, t) + \frac{\sigma_t^2}{2} \frac{t\vec{v}(\Z_t, t) - \Z_t}{1 - t} \bigg) dt + \sigma_t d\vec W_t.
$$
If we choose $\sigma_t^2 = 2c \frac{1-t}{t}$, then we get
\begin{equation}
\begin{split}
d\Z_t 
&= \bigg( \vec{v}(\Z_t, t) + c \frac{1-t}{t} \frac{t\vec{v}(\Z_t, t) - \Z_t}{1 - t} \bigg) dt + \sqrt{2c \frac{1-t}{t}} d\vec W_t \\
&= \bigg((1 + c) \vec{v}(\Z_t, t) - \frac{c}{t} \Z_t\bigg) dt + \sqrt{2c \frac{1-t}{t}} d\vec W_t,
\end{split}
\end{equation}
which matches \Eqref{eq:overshoot-sde} exactly.

\subsection{Experiment Details}~\label{sec::experiment_setting}

\paragraph{Model Configurations and Hyperparameter Settings.}
The hyperparameter settings for the Flux (FLUX.1-dev), Stable Diffusion 3 Medium, and AuraFlow models are summarized in Table~\ref{tab:hyperparameters}. Unless stated otherwise, all experiments are conducted with a default inference step count of 100.
\begin{table*}[ht!]
\centering
\small
\resizebox{0.7\textwidth}{!}{
\begin{tabular}{lccc}
\toprule
Hyperparameters & FLUX.1-dev & Stable Diffusion 3 Medium & AuraFlow \\
\midrule
Image size & $1024 \times 1024$ & $1024 \times 1024$ & $1024 \times 1024$ \\
CFG scale & 3.5 & 7.0 & 3.5 \\
Model Precision & BFloat16 & Float32 & Float16 \\
Overshooting Strength $c$ & 2.0 & 1.0 & 1.0 \\
\bottomrule
\end{tabular}
}
\caption{Hyperparameter settings for our experiments.}
\label{tab:hyperparameters}
\end{table*}

\paragraph{Human Evaluation Setup. }
We conducted human evaluations to assess text rendering quality and image fidelity. The details of the human evaluation setup are as follows:
\begin{itemize}
    \item \textbf{Text Rendering Evaluation:} The evaluation includes a total of 100 prompts, each consisting of 5–8 words, which are provided in the supplementary material (\texttt{prompts\_human\_eval.txt}). Participants are presented with a text prompt and an image generated by one of the models. They are tasked with assessing the correctness of the rendered text in the image. Each image is evaluated by at least two participants. In total, this evaluation involved 72 unique participants.
    \item \textbf{Comparative Evaluation of Text and Image Quality:} To compare text rendering quality and overall image quality, 100 samples were selected. These include 25 prompts each from the DrawTextCreative, ChineseDrawText, and TMDBEval500 benchmarks, as well as the primary human evaluation prompts. This evaluation was conducted with 15 participants. 
\end{itemize}
This comprehensive evaluation ensures a robust assessment of the model's ability to generate high-quality images and accurately render text.

\subsection{Problem in evaluating OCR}~\label{sec:evaluation_issue}
While OCR tools provide an automatic method for assessing the correctness of rendered text in images, our experiments reveal limitations in existing OCR systems when evaluating state-of-the-art text-to-image models such as FLUX. Specifically, we employed Mask TextSpotter v3~\cite{liao2020mask} and found that it struggles to accurately detect and recognize text generated by FLUX.
As illustrated in Figure~\ref{fig:ocr_result}, Mask TextSpotter performs better when evaluating models like TextDiffuser and GlyphControl, which tend to generate text with simpler layouts and standard fonts. These characteristics align more closely with the training data of the OCR model, making detection easier. In contrast, FLUX-generated text exhibits greater stylistic flexibility and diversity, posing significant challenges for existing OCR tools despite the text being rendered correctly. We provide examples in Figure~\ref{fig:ocr_result}, highlighting the OCR performance disparity. The detected text boxes and predictions are shown in red. These results underscore the need for improved OCR systems capable of handling the creative and flexible text styles generated by advanced text-to-image models.

\begin{figure}[ht!]
    \centering
    \includegraphics[width=\columnwidth]{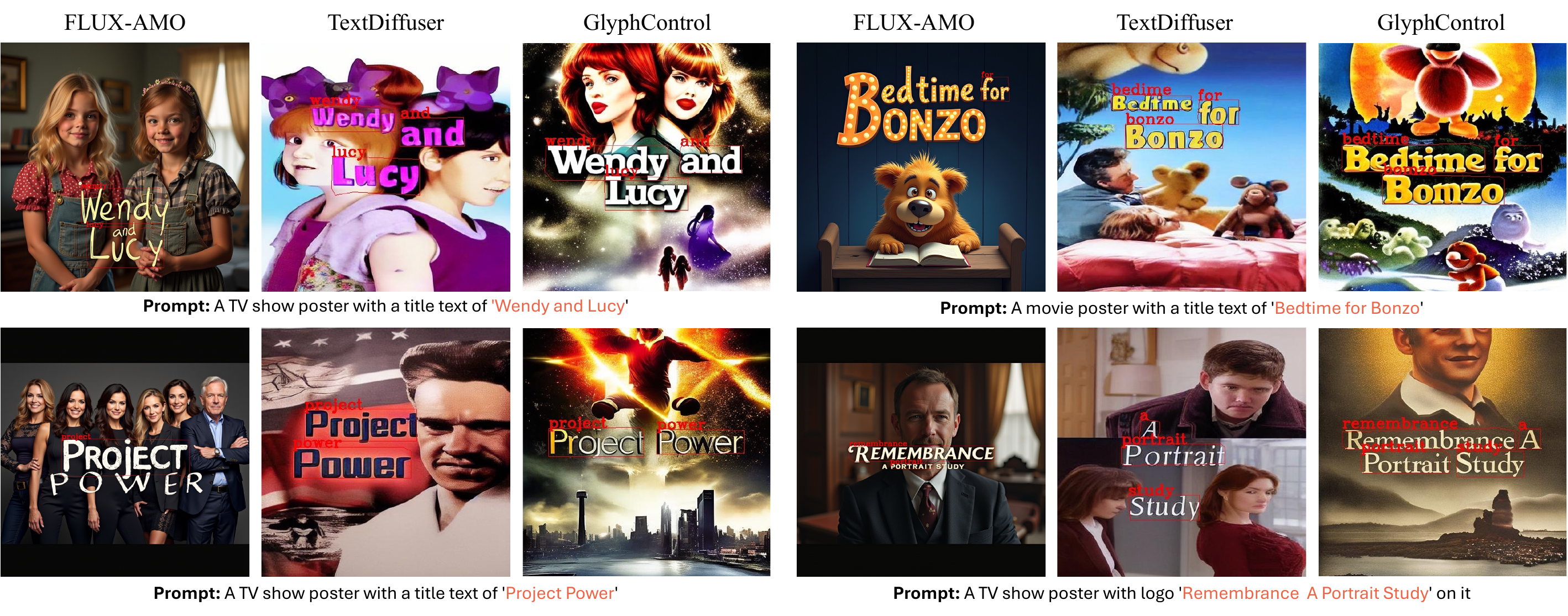}
    \caption{Examples of OCR model performance. Detected text boxes and prediction results are shown in red. The OCR model fails to detect text generated by the FLUX model effectively, even though the text is rendered correctly.}
    \label{fig:ocr_result}
\end{figure}

\begin{figure}[ht!]
    \centering
    \includegraphics[width=1.0\columnwidth]{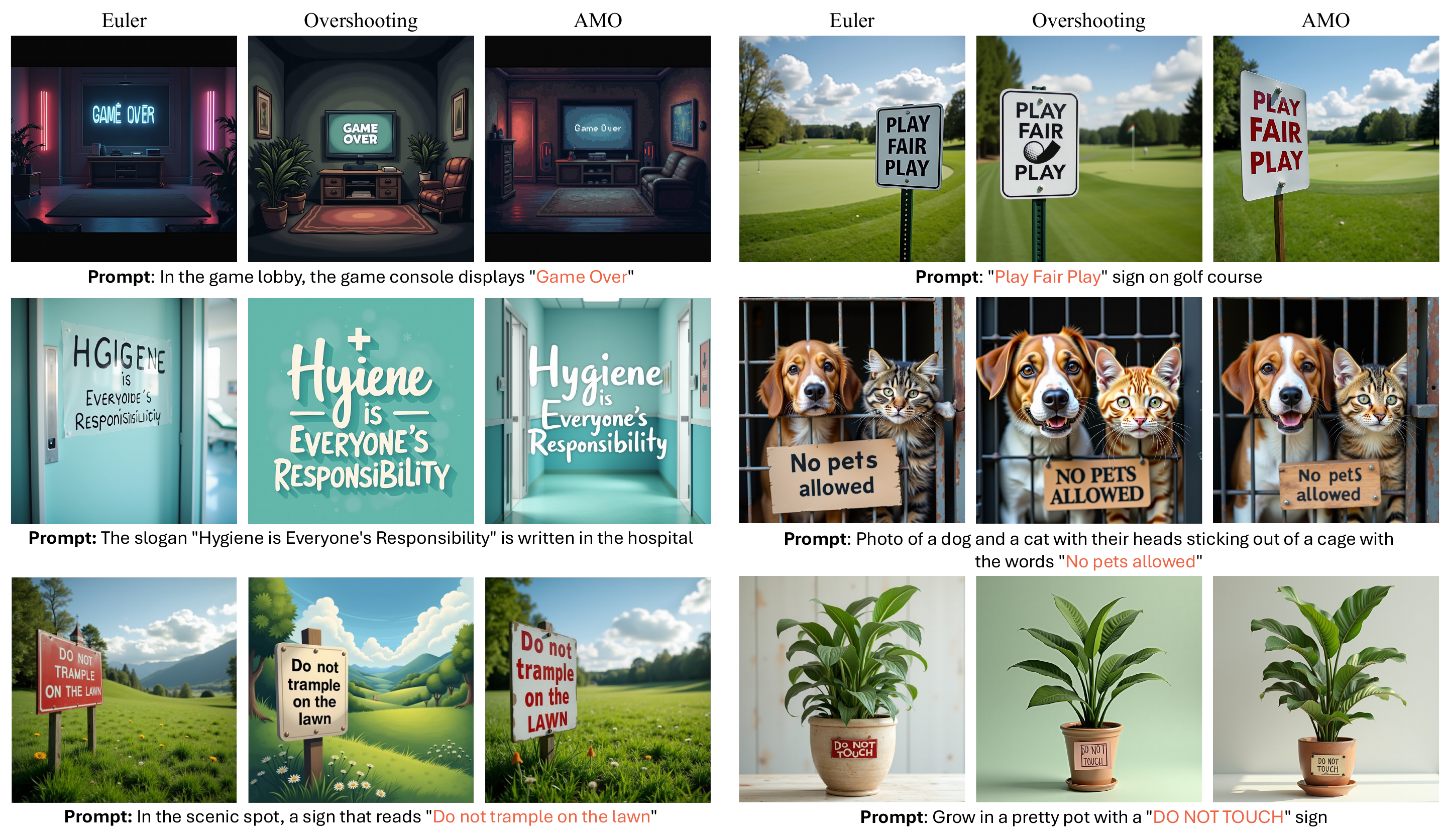}
    \caption{\textbf{Image Quality for Euler, Overshooting, and AMO. } Please zoom in for finer details. The Overshooting method shown here employs a one-step overshooting strategy, ensuring the overall computational cost remains comparable across all three methods. The overshooting approach results in cartoonish, over-smoothed outputs that lack high-frequency details. In contrast, Euler and AMO generate images that resemble real-world visuals more closely.}
    \label{fig:details_appendix}
\end{figure}

\subsection{Additional Qualitative Results}

\subsubsection{Additional Results on Image Quality for Euler, Overshooting, and AMO}

We present additional results in Figure~\ref{fig:details_appendix} to further illustrate our findings. 
These results confirm that overshooting (without attention modulation) tends to produce an over-smoothing effect, leading to generated samples lacking high-frequency details.

\begin{figure}[ht!]
    \centering
    \includegraphics[width=\columnwidth]{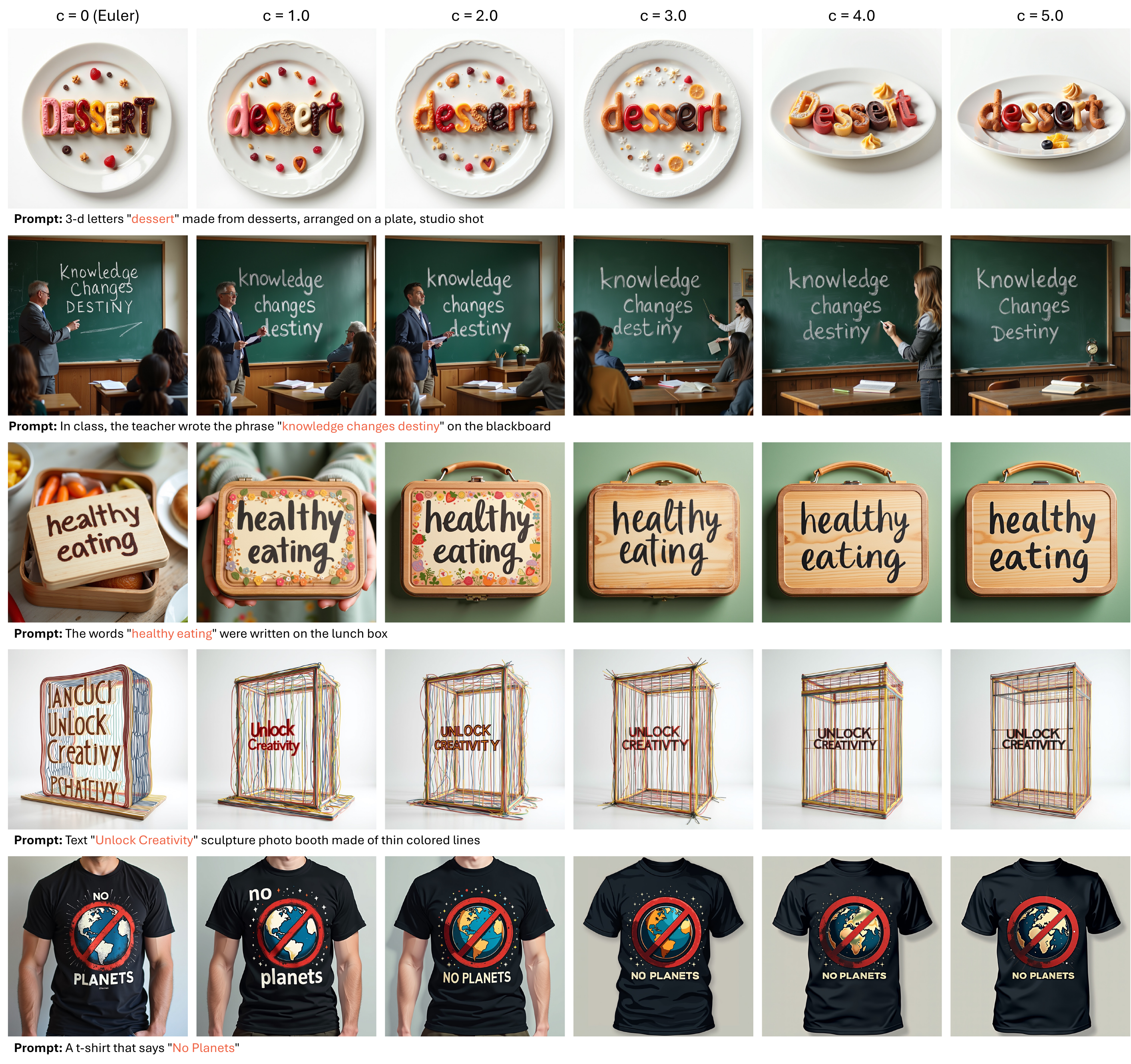}
    \caption{\textbf{Samples generated by varying $c$.} As $c$ increases, the images gradually lose complexity and fine details due to over-smoothing. For moderate values of $c$, such as $c = 2$, the results achieve a balance between accurate text rendering and visual quality.}
    \label{fig:samples_varying_c}
\end{figure}

\subsubsection{Quantative Results on AMO with Different Overshooting Strength $c$}\label{sec::large_c_artifact}

In the experiment section, we demonstrated that increasing the overshooting strength $c$ improves text rendering accuracy, with performance plateauing at $c \geq 2$ and occasionally declining for very large values of $c$. Here, we provide visual examples for varying values of $c$, as shown in Figure~\ref{fig:samples_varying_c}. We observe that as $c$ increases significantly, the generated images tend to exhibit simpler structures and fewer details. This behavior is expected because the attention modulation applies a soft overshooting strategy, where excessively large $c$ introduces over-smoothing artifacts. However, these artifacts are significantly mitigated compared to results generated without attention modulation.

\subsubsection{Additional Samples on Comparison between Euler and AMO}
We provide more results in Figure~\ref{fig:samples_euler_amo_appendix}. showcasing the differences between the Euler sampler and our AMO method.

\begin{figure}[ht!]
    \centering
    \includegraphics[width=\columnwidth]{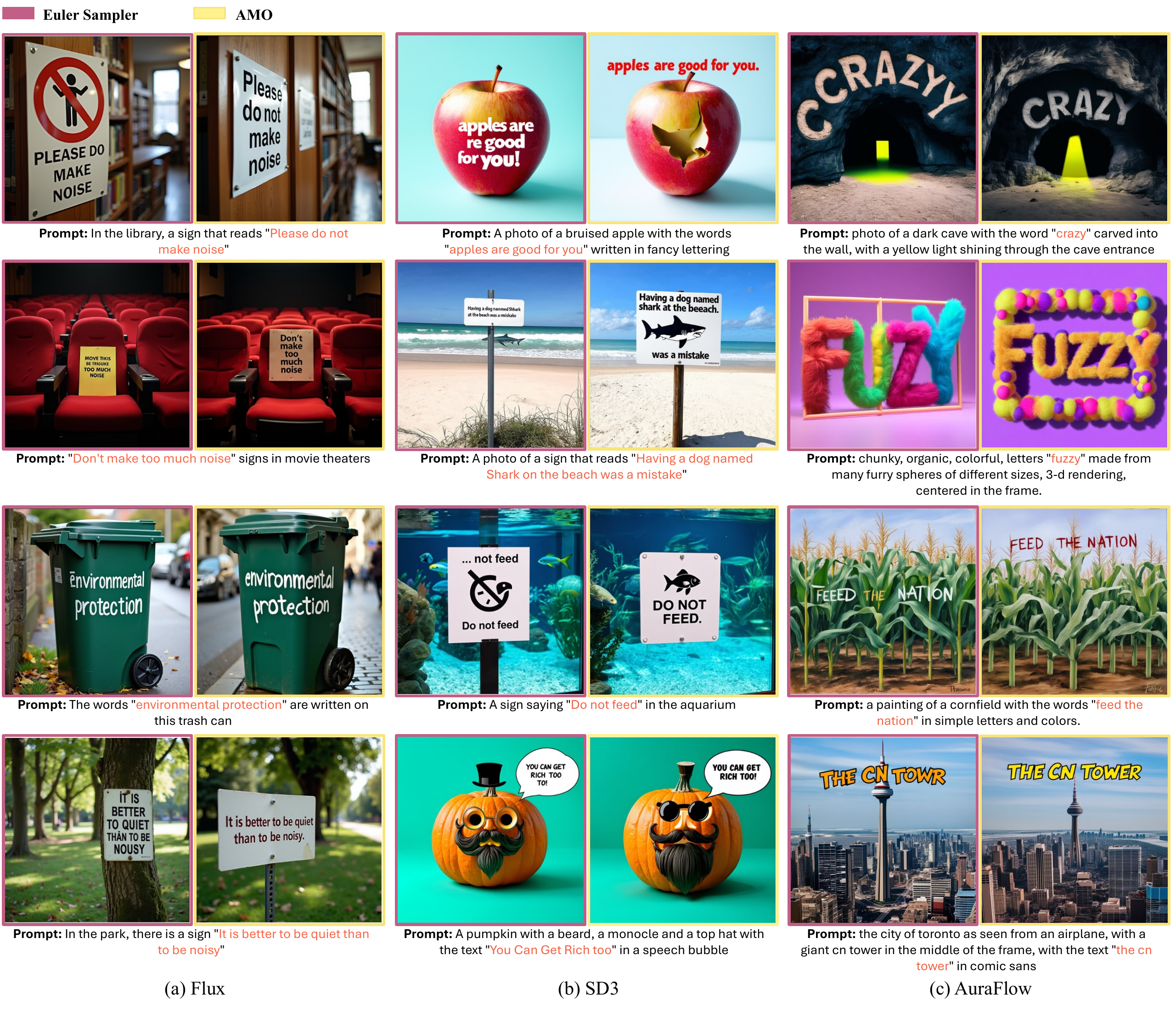}
    \caption{\textbf{Comparison of text rendering quality between Euler and AMO.} Results are presented across three different text-to-image models: Flux, Stable Diffusion 3, and AuraFlow. All images are generated using the same random seed. In each pair of images, the left column shows the results from the Euler sampler, while the right column displays results generated by our AMO method. AMO consistently produces clearer and more legible text that aligns more closely with the given prompts, demonstrating its superiority in text rendering quality.}
    \label{fig:samples_euler_amo_appendix}
\end{figure}

\clearpage
\subsection{Additional Results on Comparison with Finetuned Text-to-Image Models}~\label{sec::more_visual_results}

We present sample images from the human evaluation study comparing TextDiffuser, GlyphControl, Euler, Overshooting, and AMO. These examples are shown in Figure~\ref{fig:samples_human_study_appendix}. During the human evaluation, participants were presented with five images generated by the respective methods and asked to answer two questions: \textbf{Question 1}: Which of the following images exhibits the highest text rendering quality? (Multiple-choice) \textbf{Question 2}: Which of the following images demonstrates the best overall image quality? (Single-choice)

\begin{figure}[ht!]
    \centering
    \includegraphics[width=\columnwidth]{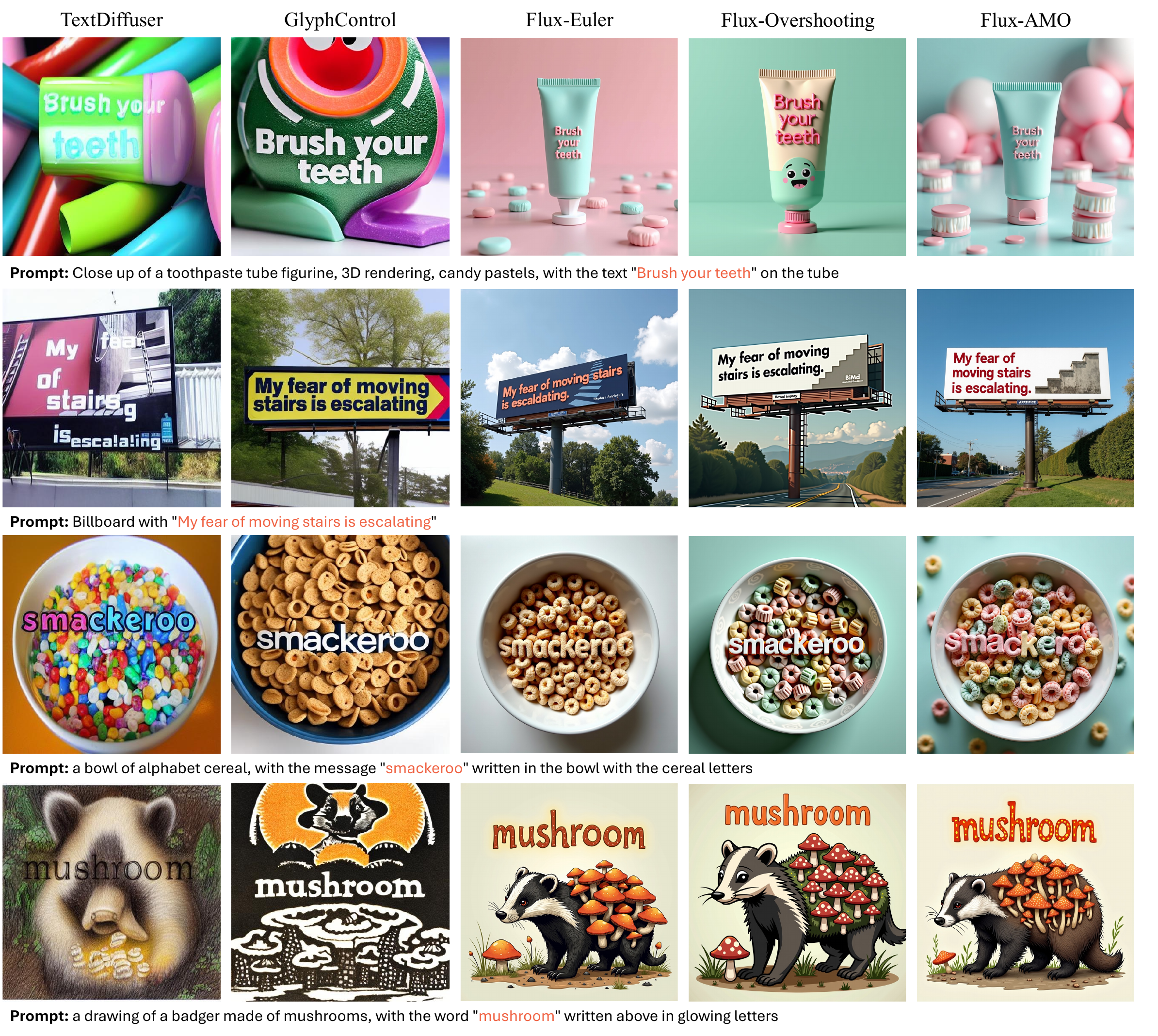}
    \caption{\textbf{Comparison of samples generated by different methods}, including TextDiffuser, GlyphControl, Euler, Overshooting, and AMO. During the human evaluation, participants were shown five images for comparison.} 
    \label{fig:samples_human_study_appendix}
\end{figure}

\noindent
\subsection{Exploring Tasks Beyond Text Rendering} Our initial exploration shows that overshooting sampler improves the rendering of details such as hands and human body structures (see Fig.~\ref{fig:hand}). However, these improvements are difficult to quantify without extensive human evaluation. Hence, we focus on text rendering, where OCR-based metrics, supplemented by human evaluation, provide a more direct and affordable evaluation. This work lays the groundwork for future exploration of other tasks.

\begin{figure}[ht!]
    \centering
    \vspace{-5pt}
    \includegraphics[width=0.9\columnwidth]{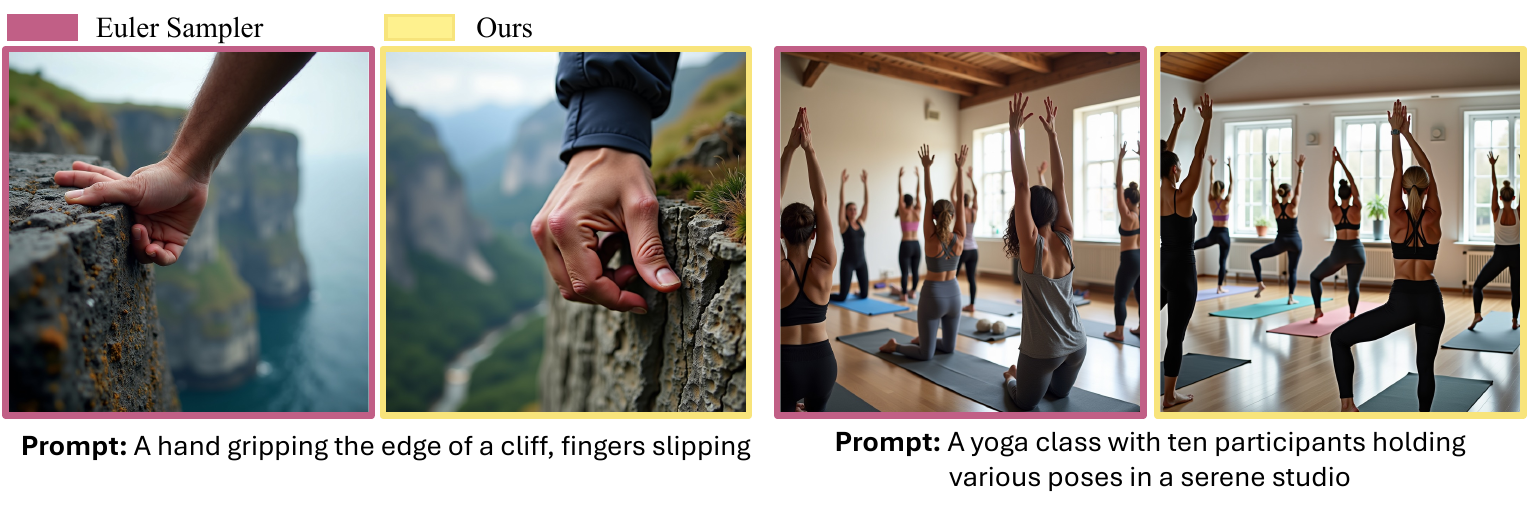}\vspace{-5pt}
    \caption{Correcting hands and body structure using our method. }
    \label{fig:hand}
\end{figure}

\noindent
\subsection{OCR-based Comparison with Specialized Text Rendering Models} We present the OCR results in Table~\ref{tab:ocr}, but it is crucial to mention that these numbers can be \textbf{misleading}. Current OCR tools struggle with the diverse and artistic fonts generated by general-purpose T2I models such as Flux.  Specialized text rendering models, on the other hand, tend to produce text in a single, OCR-optimized font. Consequently, the seemingly lower OCR scores for Flux do not necessarily indicate poorer text rendering performance. Thus we prefer human evaluation as shown in Figure~\ref{fig:human_study}. 

On the TMDBEval500 dataset (500 images), our manual evaluation of the OCR model revealed it has only 54\% accuracy in recognizing rendered text from Flux, compared to 92\% for TextDiffuser. Furthermore, extraneous text content, often generated by general-purpose T2I models, can negatively influence OCR-A by reducing precision.

\begin{table}[h!]
\centering
\resizebox{0.5\columnwidth}{!}{
\begin{tabular}{l|cccc}
\toprule
& \textbf{TextDiffuser} & \textbf{GlyphControl} & \textbf{Flux-Euler} & \textbf{Flux-Ours}\\
\midrule
\textbf{OCR-A}  & 0.491 & 0.537 & 0.313 & 0.381 \\  
\textbf{OCR-F}  & 0.625 & 0.591 & 0.458 & 0.494 \\  
\bottomrule
\end{tabular}
}
\caption{OCR-A and OCR-F results. Note that TextDiffuser and GlyphControl were optimized for OCR tools, while Flux's generated text are more diverse, leading to a lower reported score.}
\label{tab:ocr}
\end{table}

%% file: main.bbl
\begin{thebibliography}{35}
\providecommand{\natexlab}[1]{#1}
\providecommand{\url}[1]{\texttt{#1}}
\expandafter\ifx\csname urlstyle\endcsname\relax
  \providecommand{\doi}[1]{doi: #1}\else
  \providecommand{\doi}{doi: \begingroup \urlstyle{rm}\Url}\fi

\bibitem[Albergo and Vanden-Eijnden(2022)]{albergo2022building}
Michael~S Albergo and Eric Vanden-Eijnden.
\newblock Building normalizing flows with stochastic interpolants.
\newblock \emph{arXiv preprint arXiv:2209.15571}, 2022.

\bibitem[Baldridge et~al.(2024)Baldridge, Bauer, Bhutani, Brichtova, Bunner, Chan, Chen, Dieleman, Du, Eaton-Rosen, et~al.]{baldridge2024imagen}
Jason Baldridge, Jakob Bauer, Mukul Bhutani, Nicole Brichtova, Andrew Bunner, Kelvin Chan, Yichang Chen, Sander Dieleman, Yuqing Du, Zach Eaton-Rosen, et~al.
\newblock Imagen 3.
\newblock \emph{arXiv preprint arXiv:2408.07009}, 2024.

\bibitem[Betker et~al.(2023)Betker, Goh, Jing, Brooks, Wang, Li, Ouyang, Zhuang, Lee, Guo, et~al.]{betker2023improving}
James Betker, Gabriel Goh, Li Jing, Tim Brooks, Jianfeng Wang, Linjie Li, Long Ouyang, Juntang Zhuang, Joyce Lee, Yufei Guo, et~al.
\newblock Improving image generation with better captions.
\newblock \emph{Computer Science. https://cdn. openai. com/papers/dall-e-3. pdf}, 2\penalty0 (3):\penalty0 8, 2023.

\bibitem[Chen et~al.(2023)Chen, Huang, Lv, Cui, Chen, and Wei]{chen2023textdiffuser}
Jingye Chen, Yupan Huang, Tengchao Lv, Lei Cui, Qifeng Chen, and Furu Wei.
\newblock Textdiffuser-2: Unleashing the power of language models for text rendering.
\newblock \emph{arXiv preprint arXiv:2311.16465}, 2023.

\bibitem[Chen et~al.(2024)Chen, Huang, Lv, Cui, Chen, and Wei]{chen2024textdiffuser}
Jingye Chen, Yupan Huang, Tengchao Lv, Lei Cui, Qifeng Chen, and Furu Wei.
\newblock Textdiffuser: Diffusion models as text painters.
\newblock \emph{Advances in Neural Information Processing Systems}, 36, 2024.

\bibitem[Dhariwal and Nichol(2021)]{dhariwal2021diffusion}
Prafulla Dhariwal and Alexander Nichol.
\newblock Diffusion models beat gans on image synthesis.
\newblock \emph{Advances in neural information processing systems}, 34:\penalty0 8780--8794, 2021.

\bibitem[Esser et~al.(2024)Esser, Kulal, Blattmann, Entezari, M{\"u}ller, Saini, Levi, Lorenz, Sauer, Boesel, et~al.]{esser2024scaling}
Patrick Esser, Sumith Kulal, Andreas Blattmann, Rahim Entezari, Jonas M{\"u}ller, Harry Saini, Yam Levi, Dominik Lorenz, Axel Sauer, Frederic Boesel, et~al.
\newblock Scaling rectified flow transformers for high-resolution image synthesis.
\newblock In \emph{Forty-first International Conference on Machine Learning}, 2024.

\bibitem[Face(2023)]{huggingface_auraflow_2023}
Hugging Face.
\newblock \emph{Aura Flow Pipeline Documentation}, 2023.
\newblock Huggingface.

\bibitem[Fei et~al.(2024)Fei, Fan, Yu, and Huang]{fei2024flux}
Zhengcong Fei, Mingyuan Fan, Changqian Yu, and Junshi Huang.
\newblock Flux that plays music.
\newblock \emph{arXiv preprint arXiv:2409.00587}, 2024.

\bibitem[Heitz et~al.(2023)Heitz, Belcour, and Chambon]{heitz2023iterative}
Eric Heitz, Laurent Belcour, and Thomas Chambon.
\newblock Iterative $\alpha$-(de) blending: A minimalist deterministic diffusion model.
\newblock In \emph{ACM SIGGRAPH 2023 Conference Proceedings}, pages 1--8, 2023.

\bibitem[Ho and Salimans(2022)]{ho2022classifier}
Jonathan Ho and Tim Salimans.
\newblock Classifier-free diffusion guidance.
\newblock \emph{arXiv preprint arXiv:2207.12598}, 2022.

\bibitem[Ho et~al.(2020)Ho, Jain, and Abbeel]{ho2020denoising}
Jonathan Ho, Ajay Jain, and Pieter Abbeel.
\newblock Denoising diffusion probabilistic models.
\newblock \emph{Advances in neural information processing systems}, 33:\penalty0 6840--6851, 2020.

\bibitem[Ho et~al.(2022)Ho, Salimans, Gritsenko, Chan, Norouzi, and Fleet]{ho2022video}
Jonathan Ho, Tim Salimans, Alexey Gritsenko, William Chan, Mohammad Norouzi, and David~J Fleet.
\newblock Video diffusion models.
\newblock \emph{Advances in Neural Information Processing Systems}, 35:\penalty0 8633--8646, 2022.

\bibitem[Karras et~al.(2022)Karras, Aittala, Aila, and Laine]{karras2022elucidating}
Tero Karras, Miika Aittala, Timo Aila, and Samuli Laine.
\newblock Elucidating the design space of diffusion-based generative models.
\newblock \emph{Advances in neural information processing systems}, 35:\penalty0 26565--26577, 2022.

\bibitem[Labs(2023)]{blackforestlabs_flux_2023}
Black~Forest Labs.
\newblock Flux.
\newblock \url{https://github.com/black-forest-labs/flux}, 2023.
\newblock GitHub repository.

\bibitem[Le et~al.(2024)Le, Vyas, Shi, Karrer, Sari, Moritz, Williamson, Manohar, Adi, Mahadeokar, et~al.]{le2024voicebox}
Matthew Le, Apoorv Vyas, Bowen Shi, Brian Karrer, Leda Sari, Rashel Moritz, Mary Williamson, Vimal Manohar, Yossi Adi, Jay Mahadeokar, et~al.
\newblock Voicebox: Text-guided multilingual universal speech generation at scale.
\newblock \emph{Advances in neural information processing systems}, 36, 2024.

\bibitem[Liao et~al.(2020)Liao, Pang, Huang, Hassner, and Bai]{liao2020mask}
Minghui Liao, Guan Pang, Jing Huang, Tal Hassner, and Xiang Bai.
\newblock Mask textspotter v3: Segmentation proposal network for robust scene text spotting.
\newblock In \emph{Computer Vision--ECCV 2020: 16th European Conference, Glasgow, UK, August 23--28, 2020, Proceedings, Part XI 16}, pages 706--722. Springer, 2020.

\bibitem[Lipman et~al.(2022)Lipman, Chen, Ben-Hamu, Nickel, and Le]{lipman2022flow}
Yaron Lipman, Ricky~TQ Chen, Heli Ben-Hamu, Maximilian Nickel, and Matt Le.
\newblock Flow matching for generative modeling.
\newblock \emph{arXiv preprint arXiv:2210.02747}, 2022.

\bibitem[Liu(2022)]{liu2022rectified}
Qiang Liu.
\newblock Rectified flow: A marginal preserving approach to optimal transport.
\newblock \emph{arXiv preprint arXiv:2209.14577}, 2022.

\bibitem[Liu et~al.(2022{\natexlab{a}})Liu, Garrette, Saharia, Chan, Roberts, Narang, Blok, Mical, Norouzi, and Constant]{liu2022character}
Rosanne Liu, Dan Garrette, Chitwan Saharia, William Chan, Adam Roberts, Sharan Narang, Irina Blok, RJ Mical, Mohammad Norouzi, and Noah Constant.
\newblock Character-aware models improve visual text rendering.
\newblock \emph{arXiv preprint arXiv:2212.10562}, 2022{\natexlab{a}}.

\bibitem[Liu et~al.(2022{\natexlab{b}})Liu, Gong, and Liu]{liu2022flow}
Xingchao Liu, Chengyue Gong, and Qiang Liu.
\newblock Flow straight and fast: Learning to generate and transfer data with rectified flow.
\newblock In \emph{The Eleventh International Conference on Learning Representations}, 2022{\natexlab{b}}.

\bibitem[Liu et~al.(2024)Liu, Liang, Zhao, Chen, Li, and Yuan]{liu2024glyph}
Zeyu Liu, Weicong Liang, Yiming Zhao, Bohan Chen, Ji Li, and Yuhui Yuan.
\newblock Glyph-byt5-v2: A strong aesthetic baseline for accurate multilingual visual text rendering.
\newblock \emph{arXiv preprint arXiv:2406.10208}, 2024.

\bibitem[Lu et~al.(2022)Lu, Zhou, Bao, Chen, Li, and Zhu]{lu2022dpm}
Cheng Lu, Yuhao Zhou, Fan Bao, Jianfei Chen, Chongxuan Li, and Jun Zhu.
\newblock Dpm-solver: A fast ode solver for diffusion probabilistic model sampling in around 10 steps.
\newblock \emph{Advances in Neural Information Processing Systems}, 35:\penalty0 5775--5787, 2022.

\bibitem[Ma et~al.(2023)Ma, Zhao, Chen, Wang, Niu, Lu, and Lin]{ma2023glyphdraw}
Jian Ma, Mingjun Zhao, Chen Chen, Ruichen Wang, Di Niu, Haonan Lu, and Xiaodong Lin.
\newblock Glyphdraw: Seamlessly rendering text with intricate spatial structures in text-to-image generation.
\newblock \emph{arXiv preprint arXiv:2303.17870}, 2023.

\bibitem[Meng et~al.(2023)Meng, Rombach, Gao, Kingma, Ermon, Ho, and Salimans]{meng2023distillation}
Chenlin Meng, Robin Rombach, Ruiqi Gao, Diederik Kingma, Stefano Ermon, Jonathan Ho, and Tim Salimans.
\newblock On distillation of guided diffusion models.
\newblock In \emph{Proceedings of the IEEE/CVF Conference on Computer Vision and Pattern Recognition}, pages 14297--14306, 2023.

\bibitem[Podell et~al.(2023)Podell, English, Lacey, Blattmann, Dockhorn, M{\"u}ller, Penna, and Rombach]{podell2023sdxl}
Dustin Podell, Zion English, Kyle Lacey, Andreas Blattmann, Tim Dockhorn, Jonas M{\"u}ller, Joe Penna, and Robin Rombach.
\newblock Sdxl: Improving latent diffusion models for high-resolution image synthesis.
\newblock \emph{arXiv preprint arXiv:2307.01952}, 2023.

\bibitem[Polyak et~al.(2024)Polyak, Zohar, Brown, Tjandra, Sinha, Lee, Vyas, Shi, Ma, Chuang, et~al.]{polyak2024movie}
Adam Polyak, Amit Zohar, Andrew Brown, Andros Tjandra, Animesh Sinha, Ann Lee, Apoorv Vyas, Bowen Shi, Chih-Yao Ma, Ching-Yao Chuang, et~al.
\newblock Movie gen: A cast of media foundation models.
\newblock \emph{arXiv preprint arXiv:2410.13720}, 2024.

\bibitem[Radford et~al.(2021)Radford, Kim, Hallacy, Ramesh, Goh, Agarwal, Sastry, Askell, Mishkin, Clark, et~al.]{radford2021learning}
Alec Radford, Jong~Wook Kim, Chris Hallacy, Aditya Ramesh, Gabriel Goh, Sandhini Agarwal, Girish Sastry, Amanda Askell, Pamela Mishkin, Jack Clark, et~al.
\newblock Learning transferable visual models from natural language supervision.
\newblock In \emph{International conference on machine learning}, pages 8748--8763. PMLR, 2021.

\bibitem[Ramesh et~al.(2022)Ramesh, Dhariwal, Nichol, Chu, and Chen]{ramesh2022hierarchical}
Aditya Ramesh, Prafulla Dhariwal, Alex Nichol, Casey Chu, and Mark Chen.
\newblock Hierarchical text-conditional image generation with clip latents.
\newblock \emph{arXiv preprint arXiv:2204.06125}, 1\penalty0 (2):\penalty0 3, 2022.

\bibitem[Saharia et~al.(2022)Saharia, Chan, Saxena, Li, Whang, Denton, Ghasemipour, Gontijo~Lopes, Karagol~Ayan, Salimans, et~al.]{saharia2022photorealistic}
Chitwan Saharia, William Chan, Saurabh Saxena, Lala Li, Jay Whang, Emily~L Denton, Kamyar Ghasemipour, Raphael Gontijo~Lopes, Burcu Karagol~Ayan, Tim Salimans, et~al.
\newblock Photorealistic text-to-image diffusion models with deep language understanding.
\newblock \emph{Advances in neural information processing systems}, 35:\penalty0 36479--36494, 2022.

\bibitem[Song et~al.(2020{\natexlab{a}})Song, Meng, and Ermon]{song2020denoising}
Jiaming Song, Chenlin Meng, and Stefano Ermon.
\newblock Denoising diffusion implicit models.
\newblock \emph{arXiv preprint arXiv:2010.02502}, 2020{\natexlab{a}}.

\bibitem[Song and Ermon(2019)]{song2019generative}
Yang Song and Stefano Ermon.
\newblock Generative modeling by estimating gradients of the data distribution.
\newblock \emph{Advances in neural information processing systems}, 32, 2019.

\bibitem[Song et~al.(2020{\natexlab{b}})Song, Sohl-Dickstein, Kingma, Kumar, Ermon, and Poole]{song2020score}
Yang Song, Jascha Sohl-Dickstein, Diederik~P Kingma, Abhishek Kumar, Stefano Ermon, and Ben Poole.
\newblock Score-based generative modeling through stochastic differential equations.
\newblock \emph{arXiv preprint arXiv:2011.13456}, 2020{\natexlab{b}}.

\bibitem[Tuo et~al.(2023)Tuo, Xiang, He, Geng, and Xie]{tuo2023anytext}
Yuxiang Tuo, Wangmeng Xiang, Jun-Yan He, Yifeng Geng, and Xuansong Xie.
\newblock Anytext: Multilingual visual text generation and editing.
\newblock \emph{arXiv preprint arXiv:2311.03054}, 2023.

\bibitem[Zhao and Lian(2023)]{zhao2023udifftext}
Yiming Zhao and Zhouhui Lian.
\newblock Udifftext: A unified framework for high-quality text synthesis in arbitrary images via character-aware diffusion models.
\newblock \emph{arXiv preprint arXiv:2312.04884}, 2023.

\end{thebibliography}
